\pdfoutput=1
\documentclass[11pt]{article}
\usepackage{fullpage}
\usepackage{authblk}
\usepackage{hyperref}       
\usepackage{url}            
\usepackage{booktabs}       
\usepackage{nicefrac}       
\usepackage{microtype}      
\usepackage{mathtools}
\usepackage[numbers, compress]{natbib}
\newcommand{\mycite}[1]{\citet{#1}} 

\usepackage[ruled, vlined, norelsize, algo2e]{algorithm2e}

\usepackage{macros_neurIPS}
\newcommand{\fakeItem}[1][$\bullet$]{\vspace{2mm}\indent{\bf #1}~~}
\newcommand{\fakeItemAfter}{\vspace{2mm}}

\usepackage{color-edits}
\addauthor{ka}{blue}   
\addauthor{as}{magenta} 

\usepackage[title,titletoc]{appendix}

\usepackage{tocbibind}

\hypersetup{colorlinks, citecolor=red, filecolor=blue, linkcolor=blue, urlcolor=blue}
\usepackage{enumitem}

\newcommand{\xhdr}[1]{\vspace{2mm} \noindent{\bf #1}}

\newcommand{\vc}{\vec{c}}
\newcommand{\vr}{\vec{r}}
\newcommand{\vo}{\vec{o}}
\newcommand{\vx}{\vec{x}}
\newcommand{\vX}{\vec{X}}
\newcommand{\vY}{\vec{Y}}

\newcommand{\LBrewards}{quasi-rewards\xspace}

\newcommand{\myGamma}{C_{\mathtt{rad}}}
\newcommand{\myEta}{\eta_{\textsc{lp}}}

\newcommand{\BwK}{\term{BwK}}
\newcommand{\CBwK}{\term{CBwK}}
\newcommand{\LinCBwK}{\term{LinCBwK}}
\newcommand{\SemiBwK}{\term{SemiBwK}}
\newcommand{\MnlBwK}{\term{MnlBwK}}

\def \OPT {\term{OPT}}
\def \OPTDP {\OPT_{\term{DP}}} 
\def \OPTFD {\OPT_{\term{FD}}} 
\def \OPTLP {\OPT_{\term{LP}}}
\def \OPTLPSC {\OPT_{\term{LP}}^{\term{sc}}}

\def \LP {\term{LP}}

\def \ALG {\term{ALG}}
\def \REW {\term{REW}}

\def \rad {\operatorname{Rad}}
\def \fRad {f_{\term{rad}}}

\newcommand{\myArms}{a\not\in\{a^*,\nullArm\}} 
\newcommand{\logThm}{\log(Kd T)}       
\newcommand{\Nmax}{N_{\term{max}}} 
\newcommand{\relaxedBwK}[1][]
{Relaxed BwK\xspace#1}

\newcommand{\Gap}{G_\term{LP}}            

\newcommand{\DGap}{G_\term{LAG}}
\newcommand{\Dmin}{\DGap} 

\newcommand{\Val}{V}            
\newcommand{\ValP}{V_{+}}            
\newcommand{\ConfSum}{W}        
\newcommand{\AcConfSum}{\ConfSum_{\mathtt{act}}}        
\newcommand{\DisConfSum}[1][dis]{\ConfSum_{\mathtt{#1}}} 

\newcommand{\cmin}{c_{\min}}

\newcommand{\badDevSym}{\Psi}
\newcommand{\badDev}
    {\textstyle \sum_{\myArms} \DGap^{-2}(a)\cdot \logThm}

\newcommand{\fepsilon}{\epsilon \cdot \cLB^2}
\newcommand{\gepsilon}{\epsilon \cdot \cLB^2}
\newcommand{\tfepsilon}{\epsilon \cdot \tfrac{\cLB^3}{2}}
\newcommand{\gfepsilon}{\epsilon \cdot \tfrac{\cLB^3}{2}}

\newcommand{\LR}{\mathcal{F}}
\newcommand{\B}{B_{\mathtt{sc}}}

\newcommand{\supp}{\operatorname{supp}}
\newcommand{\nullArm}{\term{null}}

\newcommand{\AlgName}{\term{UcbBwK}}
\newcommand{\FinalAlgName}{\term{PrunedUcbBwK}}
\newcommand{\UCBBwK}{\term{UcbBwK}}

\title{Bandits with Knapsacks beyond the Worst Case%
\footnote{The initial version, titled "Advances in Bandits with Knapsacks", was published on {\tt arxiv.org} in Jan'20. \newline The present version (since Dec'20) improves both upper and lower bounds regarding $O(\log T)$ regret. In particular, it simplifies the algorithm and analysis in the upper bound, and fixes several issues in the lower bounds.}}
\author{Karthik A. Sankararaman\thanks{Facebook, Menlo Park, CA.  Email: karthikabinavs@gmail.com. Part of this work was done while a graduate student at University of Maryland, College Park, MD.}, }
\author{Aleksandrs Slivkins\thanks{Microsoft Research, New York, NY. Email: slivkins@microsoft.com.}}
\affil{}

\date{First version: January 2020\\This version: May 2021}

\begin{document}

\maketitle

\begin{abstract}
Bandits with Knapsacks (\BwK) is a general model for multi-armed bandits under supply/budget constraints. While worst-case regret bounds for \BwK are well-understood, we present three results that go beyond the worst-case perspective. First, we provide upper and lower bounds which amount to a \emph{full characterization} for logarithmic, instance-dependent regret rates.
Second, we consider ``simple regret" in \BwK, which tracks algorithm's performance in a given round, and prove that it is small in all but a few rounds. Third, we provide a general ``reduction" from \BwK to bandits which takes advantage of some known helpful structure, and apply this reduction to combinatorial semi-bandits, linear contextual bandits, and multinomial-logit bandits. Our results build on the \BwK algorithm from \citet{AgrawalDevanur-ec14}, providing new analyses thereof.
\end{abstract}


\addtocontents{toc}{\protect\setcounter{tocdepth}{0}}

\section{Introduction}
\label{sec:intro}

We study multi-armed bandit problems with supply or budget constraints. Multi-armed bandits is a simple model for \emph{exploration-exploitation tradeoff}, \ie the tension between acquiring new information and making optimal decisions. It is an active research area, spanning computer science, operations research, and economics. Supply/budget constraints arise in many realistic applications, \eg a seller who dynamically adjusts the prices or product assortment may have a limited inventory, and an algorithm that optimizes ad placement is constrained by the advertisers' budgets. Other motivating examples concern repeated auctions, crowdsourcing markets, and network routing.

We consider a general model called \emph{Bandits with Knapsacks} (\emph{BwK}),
which subsumes the examples mentioned above.
There are $d\geq 2$ \emph{resources} that are consumed over time,
one of which is time itself. Each resource $i$ starts out with budget $B_i$. In each round $t$, the algorithm chooses an action (\emph{arm}) $a=a_t$ from a fixed set of $K$ actions. The outcome is a vector in $[0,1]^{d+1}$: it consists of a reward and consumption of each resource.  This vector is drawn independently from some distribution over $[0,1]^{d+1}$, which depends on the chosen arm but not on the round, and is not known to the algorithm. The algorithm observes \emph{bandit feedback}, \ie only the outcome of the chosen arm. The algorithm stops at a known time horizon $T$, or when the total consumption of some resource exceeds its budget. The goal is to maximize the total reward, denoted $\REW$.


The presence of supply/budget constraints makes the problem much more challenging. First, algorithm's choices constrain what it can do in the future. Second, the algorithm is no longer looking for arms with maximal expected per-round reward (because such arms may consume too much resources). Third, the best fixed distribution over arms can be much better than the best fixed arm.
Accordingly, we compete with the \emph{best fixed distribution} benchmark: the total expected reward of the best distribution, denoted $\OPTFD$. All this complexity is already present even when $d=2$, \ie when there is only one resource other than time, and the minimal budget is
    $B = \min_i B_i = \Omega(T)$.


\BwK were introduced in \cite{BwK-focs13-conf,BwK-focs13} and extensively studied since then. The optimal worst-case regret rate is well-understood. In particular, it is
    $\tilde{\mO}(\sqrt{KT})$
when $B = \Omega(T)$.

We present several results that go beyond the worst-case perspective:

\fakeItem[1.]
We provide a full characterization for instance-dependent regret rates. In stochastic bandits, one obtains regret
    $\mO\rbr{ \tfrac{K}{\Delta} \log T}$,
where $\Delta$ is the the \emph{reward-gap}: the gap in expected reward between the best and the second-best arm.
We work out whether, when and how such results extend to \BwK.

\fakeItem[2.]
We show that \emph{simple regret}, which tracks algorithm's performance in a given round, can be small in all but a few rounds. Like in stochastic bandits, simple regret can be at least $\eps$ in at most $\tilde{\mO}(K/\eps^2)$ rounds, and this is achieved for all $\eps>0$ simultaneously.

\fakeItem[3.]
We improve all results mentioned above for a large number of arms, assuming some helpful structure.
In fact, we provide a general ``reduction" from \BwK to stochastic bandits, and apply this reduction to three well-studied scenarios from stochastic bandits.


\fakeItemAfter

Our algorithmic results focus on \UCBBwK, a \BwK algorithm from \mycite{AgrawalDevanur-ec14} which implements the ``optimism under uncertainty" paradigm and attains the optimal worst-case regret bound. We provide new analyses of this algorithm along the above-mentioned themes.

\xhdr{Related work.}
Background on multi-armed bandits can be found in books \citep{Bubeck-survey12,slivkins-MABbook,LS19bandit-book}.
\emph{Stochastic bandits} (\ie \BwK without resources) is a basic, well-understood version. The dependence on $\Delta$ and $\eps$ are optimal as stated above \citep{Lai-Robbins-85,bandits-ucb1,bandits-exp3}, and is achieved simultaneously with the optimal worst-case regret $\tildeO(KT)$, \eg in \mycite{bandits-ucb1}. Various refinements are known for $O(\log T)$ regret
\citep{bandits-ucb1,Audibert-TCS09,Honda-colt10,Garivier-colt11,Maillard-colt11-KL}.
Most relevant to this paper is
    $\mO\rbr{\sum_a \log (T)/\Delta(a)}$
regret, where $\Delta(a)$ is the gap in expected reward between arm $a$ and the best arm \citep{bandits-ucb1}. Improving regret for large / infinite number of arms via a helpful structure is a unifying theme for several prominent lines of work, \eg  linear bandits, convex bandits, Lipschitz bandits, and combinatorial (semi-)bandits.


Bandits with Knapsacks were introduced in \mycite{BwK-focs13-conf,BwK-focs13},
and optimally solved in the worst case.
Subsequent work extended \BwK to
a more general notion of rewards/consumptions \citep{AgrawalDevanur-ec14},
combinatorial semi-bandits \citep{Karthik-aistats18}, and contextual bandits \citep{cBwK-colt14,CBwK-colt16,CBwK-nips16}. Several special cases with budget/supply constraints were studied separately (and inspired a generalization to \BwK):
dynamic pricing \citep{BZ09,DynPricing-ec12,BesbesZeevi-or12,Wang-OR14},
 dynamic procurement \citep{DynProcurement-ec12,Krause-www13}, and dynamic ad allocation \citep{AdsWithBudgets-arxiv13,combes2015bandits}. The adversarial version of \BwK was studied by \citep{AdvBwK-focs19,Singla-colt20}. All this work considers worst-case regret bounds.


Several papers achieve $O(\log T)$ regret in \BwK, but with substantial caveats that we avoid.
\mycite{Wu-BwK-nips15}
assume deterministic consumption, whereas all motivating examples of \BwK require stochastic consumption correlated with rewards (\eg dynamic pricing consumes supply only if a sale happens). They posit $d=2$ and no other assumptions, whereas we show that ``best-arm optimality" is necessary with stochastic consumption.
\mycite{flajolet2015logarithmic} assume ``best-arm-optimality"  as we do (it is  implicit in their version of reward-gap). However, their algorithm inputs an instance-dependent parameter which is ``hidden" in \BwK. Moreover, their $O(\log T)$ regret bound scales with $\cmin$,  minimal expected consumption among arms (as $\cmin^{-4}$). Their worst-case regret bound is suboptimal, since it also scales with $\cmin$ (as $\cmin^{-2}$), and only applies for $d=2$.
%
\mycite{vera2019online} study a contextual version of $\BwK$ with two arms, one of which does nothing; this is meaningless when specialized to \BwK.
\mycite{Li2021}, subsequent to our initial draft on {\tt arxiv.org}, use extra parameters (other than a version of reward-gap), which yield $\geq \sqrt{T}$ regret whenever our lower bounds apply;%
\footnote{Conceptually, our assumption of ``best-arm-optimality" is replaced with another assumption: a lower bound on the positive entries of the optimal distribution $x^*$ (parameter $\chi$ in Section 3.3 of \mycite{Li2021}).}
it is unclear when all their parameters are small. 
No worst-case regret bounds are provided; their algorithm does not appear to achieve even $o(T)$ regret in the worst case.
Finally, \cite{Gyorgy-ijcai07,TranThanh-aaai10,TranThanh-aaai12,Qin-aaai13,rangi2018unifying} posit one constrained resource and $T=\infty$. This is an easier problem, \eg
the best arm is the best distribution over arms.

Another reduction from \BwK to bandits, found in \mycite{AdvBwK-focs19}, is very different from ours. It requires a much stronger premise (a regret bound against an adaptive adversary), and only yields worst-case regret bounds. Moreover, it reuses a bandit algorithm as a subroutine, whereas ours reuses a lemma.

\xhdr{Map of the paper.}
Logarithmic regret analysis for \UCBBwK is in Sections~\ref{sec:algorithm}, complementary lower bounds are presented in Section~\ref{sec:LB}. Results on simple regret are in Section~\ref{sec:simple-regret}. Extensions via confidence-sum analysis are in Section~\ref{sec:extensions}. Many of the proofs are deferred to appendices.

\section{Preliminaries} \label{sec:prelims}

The bandits with knapsacks (\BwK) problem is as follows. There are $K$ arms, $d$ resources, and $T$ rounds. Initially, each resource $j\in[d]$ is endowed with budget $B_j$.  In each round $t = 1 \LDOTS T$, an algorithm chooses an arm $a_t$, and observes an outcome vector $\vo_t = (r_t;\; c_{1,t} \LDOTS c_{d,t}) \in [0,1]^{d+1}$,
where $r_t$ is the reward, and $c_{j,t}$ is the consumption of each resource $j$. The algorithm stops when the consumption of some resource $j$ exceeds its budget $B_j$, or after $T$ rounds, whichever is sooner. We maximize the total reward, $\REW = \sum_{t=1}^\tau r_t$, where $\tau$ is the stopping time.
We focus on the stochastic version: for each arm $a$, there is a distribution $\mD_a$ over $[0,1]^{d+1}$ such that each outcome vector $\vo_t$ is an independent draw from distribution $\mD_{a_t}$ (which depends only on the chosen arm $a_t$). A problem instance consists of parameters $(K,d,T;\; B_1 \LDOTS B_d)$ and distributions $(\mD_a: \text{arms $a$})$.

Given a problem instance, the \emph{best dynamic policy} benchmark $\OPTDP$ maximizes the total expected reward over all algorithms; it is used in all worst-case regret bounds. The \emph{best fixed distribution} benchmark $\OPTFD$, used in some of our results, maximizes the total expected reward over all algorithms that always sample an arm from the same distribution.
%
The worst-case optimal regret rate is \citep{BwK-focs13}:
\begin{align}\label{eq:regret-opt}
\OPTDP - \E[\REW] = \tilde{\mathcal{O}}(\;\sqrt{K\, \OPTDP} + \OPTDP\sqrt{\nicefrac{K}{B}}\;),
\quad \textstyle B = \min_{j\in[d]} B_j.
\end{align}
\noindent\textbf{Simplifications and notation.}
Following prior work, we make three assumptions without losing generality. First, all budgets are the same: $B_1 = \ldots = B_d = B$. This is w.l.o.g. because one can divide the consumption of each resource $j$ by $B_j/\min_i B_i$; dependence on the budgets is driven by the smallest $B_j$. Second, resource $d$ corresponds to time: each arm deterministically consumes $B/T$ units of this resource in each round. It is called the \emph{time resource} and denoted \term{time}. Third, there is a \emph{null arm}, denoted \nullArm, whose reward and consumption of all resources except \term{time} is always $0$.%
\footnote{Choosing the null arm is equivalent to skipping a round. One can take an algorithm $\ALG$ that uses \nullArm, and turn it into an algorithm that doesn't: when $\ALG$ chooses \nullArm, just call it again until it doesn't.}

Like most prior work on \BwK, we use $\mO(\cdot)$ notation rather than track explicit constants in regret bounds. This improves clarity and emphasizes the more essential aspects of analyses and results.



For $n\in \N$, let
    $[n] = \{1 \LDOTS n\}$
and
    $\Delta_n = \{\text{all distributions on $[n]$} \}$.
Let $[K]$ and $[d]$ be, resp., the set of all arms and the set of all resources. For each arm $a$, let $r(a)$ and $c_j(a)$ be, resp., the mean reward and mean resource-$j$ consumption, \ie  $(r(a); c_1(a) \LDOTS c_d(a)) := \E_{\vo\sim \mD_a}[\vo].$
We sometimes write
    $ \vr = (r(a):\,a\in [K])$
and
    $\vc_j = (c_j(a):\,a\in [K])$
as vectors over arms. Given a function $f:[K]\to \R$, we extend it to distributions $\vX$ over arms as $f(\vX) := \E_{a\sim\vX}[f(a)] $.


\xhdr{Linear Relaxation.}
Following prior work, we consider a linear relaxation:
	\begin{equation}
		\label{lp:primalAbstract}
		\begin{array}{ll@{}ll}
		\text{maximize} \qquad
		&
		 \vX \cdot  \vec{r} & \text{such that}\\
        & \vX \in [0,1]^K,\;
 		  \vX \cdot \vec{1} = 1  &\\
\displaystyle \forall j \in [d]  \qquad
&  \vX \cdot \vec{c}_j \leq B/T .
\end{array}
\end{equation}
Here $\vX$ is a distributions over arms, the algorithm does not run out of resources in expectation, and the objective is the expected per-round reward. Let $\OPTLP$ be the value of this linear program. Then
    $\OPTLP\geq \OPTDP/T \geq \OPTFD/T$
\citep{BwK-focs13}.
The Lagrange function
    $\mL: \Delta_K \times \R^d_+ \rightarrow \mathbb{R}$ defined as follows:
		\begin{align}\label{eq:LagrangianGeneral}	
			\textstyle \mL(\vX, \vec{\lambda})
				:=  r(\vX) +
		    \sum_{j \in [d]} \lambda_j
    		[\; 1-\nicefrac{T}{B}\; c_j(\vX), \;].
		\end{align}
where $\vec{\lambda}$ corresponds to the dual variables. Then (\eg by Theorem D.2.2 in \cite{ben2001lectures}):
\begin{align}\label{eq:LagrangeMinMax}	
\min_{\vec{\lambda}\geq 0} \max_{\vec{X} \in \Delta_K}
    \mL(\vec{X}, \vec{\lambda})
= \max_{\vec{X} \in \Delta_K} \min_{\vec{\lambda}\geq 0}
    \mL(\vec{X}, \vec{\lambda})
= \OPTLP.
	\end{align}
The $\min$ and $\max$ in \eqref{eq:LagrangeMinMax} are attained, so that $(\vX^*,\vec{\lambda}^*)$ is maximin pair if and only if it is minimax pair; such pair is called a \emph{saddle point}. We'll use $\mL(\,\cdot\,,\vec{\lambda}^*)$ to generalize reward-gap to \BwK.

\xhdr{Algorithm \UCBBwK.}
We analyze an algorithm from \mycite{AgrawalDevanur-ec14}, defined as follows. In the LP \eqref{lp:primalAbstract}, rescale the last constraint, for each resource $j\neq \term{time}$, as
$(\nicefrac{B}{T})(1-\myEta)$, where
\begin{align}\label{eq:prelims-eta}
    \myEta := 3 \cdot (\;
         \sqrt{\nicefrac{K}{B}\;\logThm}
        + \nicefrac{K}{B}\; (\logThm)^2
    \;).
\end{align}
We call it the \emph{rescaled LP} (see~\eqref{lp:rescaledLP}).
Its value is $(1-\myEta)\;\OPTLP$.
At each round $t$, the algorithm forms an ``optimistic" version of this LP, upper-bounding  rewards and lower-bounding consumption:
	\begin{equation}
		\label{lp:UCBBwK}
		\begin{array}{ll@{}ll}
		\text{maximize} \qquad
		&
		\sum_{a\in [K]} X(a)\;  r^{+}_t(a) & \text{such that}\\
 		& \vX\in [0,1]^K,\quad\sum_{a\in [K]} X(a) = 1  \\
		\displaystyle \forall j \in [d]  \qquad
		& \sum_{a\in[K]} X(a)\; c^{-}_{j, t}(a) \leq B (1 - \myEta)/T .
\end{array}
\end{equation}
\UCBBwK solves~\eqref{lp:UCBBwK}, obtains distribution $\vX_t$, and samples an arm $a_t$ independently from $\vX_t$. The algorithm achieves the worst-case optimal regret bound in~\eqref{eq:regret-opt}. The upper/lower confidence bounds
    $r_t^+(a),\,c^{-}_{j, t}(a)\in[0,1]$
are computed in a particular way specified in Appendix~\ref{app:spec}. What matters to this paper is that they satisfy a high-probability event
\begin{align}\label{eq:cleanEvent}
0\leq r_t^+(a) -r(a) \leq \rad_t(a)  \;\text{and}\;
0\leq c_j(a) - c^-_{j,t}(a)\leq \rad_t(a),
\end{align}
for some \emph{confidence radius}
    $\rad_t(a)$
specified below. This event holds, simultaneously for all arms $a$, resources $j$ and rounds $t$, with probability (say) at least $1-\frac{\logThm}{T^4}$.  For $a\neq\nullArm$, we can take
\begin{align}\label{eq:MaxConfRadUB}
\rad_t(a)
   =\min(\;1,\;\sqrt{\myGamma/N_t(a)} + \myGamma/N_t(a)\;),
\end{align}
where $\myGamma = 3 \cdot \log (KdT)$ and $N_t(a)$ is the number of rounds before $t$ in which arm $a$ has been chosen.
There is no uncertainty on the time resource and the null arm, so we define
    $c_{\term{time},\,t}^{-}(\cdot) = B/T$
and
    $\rad_t(\nullArm) = r^+_t(\nullArm) = c_{j,t}^{-}(\nullArm) = 0$
for all resources $j\neq \term{time}$.

\section{Logarithmic regret bounds} \label{sec:algorithm}
We provide upper and lower bounds which amount to \emph{full characterization} of logarithmic, instance-dependent regret rates in \BwK. We achieve $O(\log T)$ regret under two assumptions: there is only one resource other than time (\ie $d=2$), and the best distribution over arms reduces to the best fixed arm (\emph{best-arm-optimality}). We prove that both assumptions are essentially necessary for any algorithm, deriving complementary $\Omega(\sqrt{T})$ lower bounds if either assumption fails. Both lower bounds hold in a wide range of problem instances; arguably, they represent typical scenarios rather than exceptions. All upper and lower bounds are against the best fixed distribution benchmark ($\OPTFD$).

We achieve  $O(\log T)$ regret with \UCBBwK algorithm \citep{AgrawalDevanur-ec14}, which implies two very desirable properties: the algorithm does not know in advance whether best-arm-optimality holds, and attains the optimal worst-case regret bound for all instances, best-arm-optimal or not. The positive result would have been   weaker without either property, although still non-trivial.

We identify a suitable instance-dependent parameter, defined via Lagrangians from \refeq{eq:LagrangianGeneral}:
\begin{equation}\label{eq:Glag}
\DGap(a) :=  \OPTLP - \mL(a, \vec{\lambda}^*)
\qquad\text{\emph{(Lagrangian gap of arm $a$)}},
\end{equation}
where $\vec{\lambda}^*$ is a minimizer in \refeq{eq:LagrangeMinMax}.
It is a non-obvious generalization of the \emph{reward-gap} from multi-armed bandits,
    $\Delta(a) = \max_{a'} r(a') - r(a) $.
The Lagrangian gap of a problem instance is
\begin{align}\label{eq:Glag-min}
     \Dmin := \textstyle \min_{\myArms} \DGap(a).
\end{align}
Our regret bound scales as
    $\mO(K\Dmin^{-1}\,\log T)$,
which is optimal in $\Dmin$, under a mild additional assumption, and as
$\mO(K\Dmin^{-2}\,\log T)$ otherwise.

\subsection{$O(\log T)$ regret analysis for \UCBBwK}

We analyze a version of \UCBBwK which ``prunes out" the null arm, call it \FinalAlgName.
(This modification can only improve regret, so it retains the worst-case regret \eqref{eq:regret-opt} of \UCBBwK.)
We provide a new analysis of this algorithm for $d=2$ and best-arm-optimality. We analyze the sensitivity of the ``optimistic" linear relaxation to small perturbations in the coefficients, and prove that the best arm is chosen in all but a few rounds. The key is to connect each arm's confidence term with its Lagrangian gap. This gives us
    $\mO(K\Dmin^{-2}\,\log T)$
regret rate. To improve it to
    $\mO(K\Dmin^{-1}\,\log T)$,
we use a careful counting argument which accounts for rewards and consumption of non-optimal arms.

Algorithm \FinalAlgName is formally defined as follows: in each round $t$, call \AlgName as an oracle, repeat until it chooses a non-null arm $a$, and set $a_t=a$. (In one ``oracle call", \AlgName outputs an arm and inputs an outcome vector for this arm.) The total number of oracle calls is capped at
    $\Nmax = \alpha_0 \cdot T^2\;\log T$,
with a sufficiently large absolute constant $\alpha_0$ which we specify later in Claim~\ref{cl:null-rounds}. Formally, after this many oracle calls the algorithm can only choose the null arm.


\begin{definition}\label{def:best-arm-optimal}
An instance of \BwK is called \emph{best-arm-optimal} with best arm $a^*\in [K]$ if the following conditions hold:
 (i)
    $\OPTLP = \tfrac{B}{T}\cdot r(a^*)  / \max_{j\in [d]} c_j(a^*)$,
(ii) the linear program \eqref{lp:primalAbstract} has a unique optimal solution $\vX^*$ supported on $\{a^*,\nullArm\}$, and (iii)
    $X^*(a^*) > \frac{3\sqrt{B} \log(KdT)}{T}$.
\end{definition}

\noindent Part (ii) here is essentially w.l.o.g.;%
\footnote{Part (ii) holds almost surely given part (i) if one adds a tiny noise, \eg $\eps$-variance, mean-$0$ Gaussian for any $\eps>0$, independently to each coefficient in the LP \eqref{lp:primalAbstract}, as per Prop. 3.1 in \mycite{megiddo1988perturbation}. To implement this, an algorithm can precompute the noise terms and add them consistently to observed rewards and consumptions.} part (iii) states that the optimal value should not be tiny.




We assume $d=2$ and best-arm-optimality throughout this section without further mention.  In particular, the linear program ~\eqref{lp:primalAbstract} has a unique optimal solution $\vX^*$, and its support has only one arm $a^*\neq \nullArm$. We use $c(a)$ to denote the mean consumption of the non-time resource on arm $a$. We distinguish two cases, depending on whether $c(a^*)$ is very close to $\nicefrac{B}{T}$.

\begin{theorem}\label{thm:logRegretUpper}
Fix a best-arm optimal problem instance with only one resource other than time (\ie $d=2$). Consider Algorithm \FinalAlgName with parameter
    $\myEta\leq \tfrac12$ in \eqref{eq:prelims-eta}.
Then
\begin{itemize}
\item[(i)] 
$			\OPTFD - \E[\REW]
    \leq  \mO\rbr{\tfrac{\OPTFD}{B} \cdot \badDevSym}$,
where $\badDevSym := \badDev$.

\item[(ii)]
Moreover, if
    $|c(a^*)-\nicefrac{B}{T}| >\Omega(\badDevSym/T)$,
then
	\begin{align}\label{eq:logRegretMainEq}
	\OPTFD - \E[\REW] \textstyle
    \leq \mO(\;
        \sum_{\myArms} \DGap^{-1}(a)\; \logThm\;).
    \end{align}
\end{itemize}
\end{theorem}

\refeq{eq:logRegretMainEq} optimally depends on $\DGap(\cdot)$: indeed, it does in the unconstrained case when Lagrangian gap specializes to the reward gap, as per the lower bound in \mycite{Lai-Robbins-85}. In particular, \refeq{eq:logRegretMainEq} holds if $\DGap>T^{-1/4}$ and
    $|c(a^*)-\nicefrac{B}{T}| >\mO(T^{-1/2})$.
The constant in $\mO(\cdot)$ is $48$ in both parts of the theorem; the analysis only suppresses constants from concentration bounds and from Lemma~\ref{cl:sensitivity-body}.


\OMIT{ 
\begin{remark}
Theorem~\ref{thm:logRegretUpper} holds even if budget $B$ is small, whereas the worst-case regret bound \eqref{eq:regret-opt} is vacuous. In particular, suppose all arms $\myArms$ have small mean reward $r(a)$ and large mean consumption $c(a)$. Then their Lagrangian gap is small, because from \eqref{eq:gLagSimplified} we have that it is an increasing function of $c(a)$, and decreasing function of $r(a)$ and $B$. So they do not contribute much to regret.
\end{remark}
} 

\OMIT{ 
\begin{theorem}[worst-case]\label{thm:worst-case}
Assume $\myEta\leq \tfrac12$ in \eqref{eq:prelims-eta}.
Then \FinalAlgName achieves
\begin{align}\label{eq:WSRegret}
\OPTDP - \E[\REW]	
\leq \mO\left( \sqrt{\logThm}
    \left( \OPTDP\sqrt{\nicefrac{K}{B}}  + \sqrt{K\,\OPTDP} \right)
        \right).
\end{align}
\end{theorem}

\begin{remark}\label{rem:logRegret-worstCase}
\refeq{eq:WSRegret} is the worst-case optimal regret bound \eqref{eq:regret-opt} and follows directly from the analysis in \cite{AgrawalDevanur-ec14}. \kacomment{We don't need to prove anything here.}
\end{remark}
} 

%

\subsubsection{Basic analysis: proof of Theorem~\ref{thm:logRegretUpper}(i)}	


We analyze $\AlgName$ in a relaxed version of \BwK, where an algorithm  runs for exactly $\Nmax$ rounds, regardless of the time horizon and the resource consumption; call it \emph{\relaxedBwK}. The algorithms are still parameterized by the original $B,T$, and observe the resource consumption.

We sometimes condition on the high-probability event that \eqref{eq:cleanEvent} holds for all rounds $t\in [\Nmax]$, call it the ``clean event". Recall that its probability is at least
    $1-\frac{\mO(\logThm)}{T^2}$.

We prove that the best arm $a^*$ chosen in all but a few rounds. The crux is an argument about sensitivity of linear programs to perturbations. More specifically, we argue about sensitivity of the support of the optimal solution for the linear relaxation $\eqref{lp:primalAbstract}$.

\begin{lemma}[LP-sensitivity]
\label{cl:sensitivity-body}
Consider an execution of \AlgName in \relaxedBwK.
Under the ``clean event",
$\rad_t(a)\geq \tfrac{1}{4}\,\DGap(a)$
for each round $t$ and each arm
    $a\in \supp(\vX_t)\setminus\{a^*,\nullArm\}$.
\end{lemma}

\begin{myproof}[Sketch]
We use a standard result about LP-sensitivity, the details are spelled out in Appendix~\ref{app:LP-sensitivity}. We apply this result via the following considerations. We treat the optimistic LP \eqref{lp:UCBBwK}  a perturbation of (the rescaled version of) the original LP $\eqref{lp:primalAbstract}$. We rely on perturbations being ``optimistic" (\ie upper-bounding rewards and lower-bounding resource consumption). We use the clean event to upper-bound the perturbation size by the confidence radius. Finally, we prove that
\begin{align}\label{eq:gLagSimplified}
\Dmin(a) =\textstyle \frac{T}{B} \sum_{j \in [d]} \;\lambda^*_j c_j(a) - r(a),
\end{align}
and use this characterization to connect Lagrangian gap to the allowed perturbation size.
\end{myproof}

We rely on the following fact which easily follows from the definition of the confidence radius:

\begin{claim}\label{cl:ConfSum-main}
Consider an execution of some algorithm in \relaxedBwK. Fix a threshold $\theta >0$. Then each arm  $a \neq \nullArm$ can only be chosen in at most
    $\mO\rbr{ \theta^{-2}\logThm }$
rounds $t$ with $\rad_t(a)\geq \theta$.
\end{claim}

\begin{corollary}\label{cor:support-optimal}
Consider an execution of \AlgName in \relaxedBwK. Under the clean event, each arm $a\not\in \{a^*,\nullArm\}$ is chosen in at most
    $N_0(a) := \mO\rbr{\DGap^{-2}(a)\;\logThm}$
rounds.
\end{corollary}


\OMIT{ 
\begin{myproof}[of Lemma~\ref{cor:support-optimal}]
Consider a time-step $t$ when an arm $a\in \supp(\vX_t)\setminus\{a^*,\nullArm\}$ is sampled. Then $\rad_t(a)\geq \tfrac{\DGap(a)}{4}$ by Lemma~\ref{cl:sensitivity-body}. By Claim~\ref{cl:ConfSum-main} we have that there are at most  $\mO\rbr{ \DGap(a)^{-2}\logThm }$ such rounds for a given arm $a \in[K]$. Thus, we get the claim.
\end{myproof}
} 


This follows from Lemma~\ref{cl:sensitivity-body} and Claim~\ref{cl:ConfSum-main}. Next, the null arm is not chosen too often:

\begin{claim}\label{cl:null-rounds}
Consider an execution of \AlgName in \relaxedBwK. With probability at least $1-\mO(T^{-3})$, the following happens: the null arm cannot be chosen in any $\alpha_0 \,T\,\log(T)$ consecutive rounds, for a large enough absolute constant $\alpha_0$. Consequently, a non-null arm is chosen in at least $T$ rounds.
\end{claim}

\begin{myproof}[Sketch]
Fix round $t$, and suppose \AlgName chooses the null arm in $N$ consecutive rounds, starting from $t$. No new data is added, so the optimistic LP stays the same throughout. Consequently, the solution $\vX_t$ stays the same, too. Thus, we have $N$ consecutive independent draws from $\vX_t$ that return $\nullArm$. It follows that $r(\vX_t) <\nicefrac{1}{T}$ with high probability, \eg by \eqref{eq:lem:radBound}. On the other hand, assume the clean event. Then $r(\vX_t) \geq (1-\myEta)\;\OPTLP$ by definition of the optimistic LP, and consequently
    $r(\vX_t) \geq (1-\myEta)\, \OPTDP/T$.
We obtain a contradiction.
\end{myproof}

Corollary~\ref{cor:support-optimal} and Claim~\ref{cl:null-rounds} imply a strong statement about the pruned algorithm.


\begin{claim}\label{cl:FinalAlg}
Consider an execution of \FinalAlgName in the (original) \BwK problem. With probability at least $1-\mO(T^{-2})$, each arm $a\not\in\{a^*,\nullArm\}$ is chosen in at most $N_0(a)$ rounds, and arm $a^*$ is chosen in $T-N_0$ remaining rounds,
    $N_0 := \sum_{\myArms}N_0(a)$.
\end{claim}

We take a very pessimistic approach to obtain Theorem~\ref{thm:logRegretUpper}(i): we only rely on rewards collected by arm $a^*$, and we treat suboptimal arms as if they bring no reward and consume the maximal possible amount of resource. We formalize this idea as follows (see Appendix~\ref{appx:logRegretSection} for details).

For a given arm $a$, let $\REW(a)$ be the total reward collected by arm $a$ in \FinalAlgName. Let $\REW(a\mid B_0, T_0)$ be the total reward of an algorithm that always plays arm $a$ if the budget and the time horizon are changed to $B_0\leq B$ and $T_0 \leq T$, respectively. Note that
\begin{align}\label{eq:logT-LPformula}
\LP(a\mid B_0,T_0)
    := \E[\REW(a\mid B_0, T_0)]
    = r(a)\cdot \min(\; T_0,\tfrac{B_0}{c(a)} \;).
\end{align}
is the value of always playing arm $a$ in a linear relaxation with the same constraints. By best-arm-optimality, we have
    $\E[\REW(a^*\mid B, T)] = \OPTFD$.
We observe that
\begin{align}\label{cl:change-B}
\E[\REW(a^*\mid B_0, T_0)]
    \geq \tfrac{\min \left\{ T_0, B_0 \right\}}{B}\; \cdot \OPTFD.
 \end{align}

By Claim~\ref{cl:FinalAlg} there are at least $B_0 = B-N_0$ units of budget and at least $T_0 = T-N_0$ rounds left for arm $a^*$ with high probability. Consequently,
\begin{align}\label{eq:logT-basic-reduce}
\E[\REW] \geq
\E[\REW(a^*)] \geq \E[\REW(a^* \mid B_0,T_0) ] - \tilde{\mO}(\nicefrac{1}{T}).
\end{align}
We obtain Theorem~\ref{thm:logRegretUpper}(i) by plugging these $B_0, T_0$ into \refeq{cl:change-B}, and then using \eqref{eq:logT-basic-reduce}.


\subsubsection{Tighter computation: proof of Theorem~\ref{thm:logRegretUpper}(ii)}

We re-use the basic analysis via Claim~\ref{cl:FinalAlg}, but perform the final computation more carefully so as to account for the rewards and resource consumption of the suboptimal arms.

Let's do some prep-work. First, we characterize $\REW(a^*)$ in a more efficient way compared to \refeq{eq:logT-basic-reduce}. Let $B(a), T(a)$ denote, resp., the budget and time consumed by \FinalAlgName when playing a given arm $a$. We use expectations of
    $B(a)$ and $T(a)$,
rather than lower bounds:
\begin{align}
\E[\REW(a)]
    & = r(a)\,\E[T(a)] = r(a)\,\tfrac{\E[B(a)]}{c(a)} & \nonumber \\
    &= \LP\rbr{ a\mid \E[B(a)], \E[T(a)] }
    &\quad\text{for each arm $a$}.
    \label{eq:expectedBudgets}
\end{align}
We prove \refeq{eq:expectedBudgets} via martingale techniques, see Appendix~\ref{app:martingale-arguments}.


Second, we use a tighter version of \refeq{cl:change-B} (see Appendix~\ref{appx:logRegretSectionStronger}): for any $B_0\leq B$, $T_0\leq T$
\begin{align}\label{cl:change-B-new}
\LP(a^*\mid B_0, T_0)] \geq
    \OPTFD \cdot\tfrac{B_0}{B} \;/
     \rbr{ \max\cbr{\tfrac{B}{T}, c(a^*)} \cdot
     \max\cbr{\tfrac{B_0}{T_0}, c(a^*)} }.
\end{align}
%
%
Third, we lower-bound $\DGap(a)$ in a way that removes Lagrange multipliers $\lambda^*$:
\begin{equation}
\label{eq:gLagP}
\Dmin(a)
    \geq
\begin{cases}
\OPTFD/T - r(a)
    &\text{if}\quad
    c(a^*) < \nicefrac{B}{T},\\
 \OPTFD\cdot c(a)/B - r(a)	
    &\text{if}\quad
    c(a^*) > \nicefrac{B}{T}.
\end{cases}
\end{equation}
We derive this from \refeq{eq:gLagSimplified} and complementary slackness, see Appendix~\ref{app:gLagP}.


	

Fourth, let
    $B_0 = \E[B(a^*)]$ and $T_0 =\E[T(a^*)]$
denote, resp., the expected budget and time consumed by arm $a^*$.
Let $N(a) = \E[T(a)]$ be the expected number of pulls for each arm
    $a\not\in \{a^*,\nullArm\}$.
In this notation, \refeq{eq:expectedBudgets} implies that
\begin{align}\label{eq:expectedBudgets-total}
    \E[\REW] =\textstyle  \sum_{\myArms} N(a)\,r(a) + \LP(a^*\mid B_0,T_0).
\end{align}

Now we are ready for the main computation . We consider four cases, depending on how $c(a^*)$ compares with $\nicefrac{B}{T}$ and $\nicefrac{B_0}{T_0}$. We prove the desired regret bound when $c(a^*)$ is either larger than both or smaller than both, and we prove that it cannot lie in between. The ``in-between" cases is the only place in the analysis where we use the assumption that $c(a^*)$ is close to $\nicefrac{B}{T}$.

\xhdr{Case 1: $c(a^*)< \min(\nicefrac{B}{T},\nicefrac{B_0}{T_0})$}.
Plugging in \refeq{cl:change-B-new} into \refeq{eq:expectedBudgets-total} and simplifying,
\begin{align}
\E[\REW]
   &\geq \textstyle\sum_{\myArms} \; N(a)\, r(a) + \OPTFD\cdot \nicefrac{T_0}{T}.
	\end{align}
	
\noindent Re-arranging, plugging in $T_0 = T-\sum_{a\neq a^*} N(a)$ and simplifying, we obtain
\begin{align}
\OPTFD - \E[\REW]
	& \textstyle \leq \sum_{\myArms} N(a) \left( \frac{\OPTFD}{T} - r(a) \right) \\ 
		& \textstyle \leq  \sum_{\myArms} N(a)\, \Dmin(a)
        & \EqComment{by \refeq{eq:gLagP}} \nonumber \\
		& \textstyle \leq \mO (\; \sum_{\myArms} \DGap^{-1}(a)\; \logThm \;)
    &\EqComment{by Claim~\ref{cl:FinalAlg}}. \nonumber
	\end{align}

\xhdr{Case 2: $c(a^*)> \max(\nicefrac{B}{T},\nicefrac{B_0}{T_0})$}.
Plugging in \refeq{cl:change-B-new} into \refeq{eq:expectedBudgets-total} and simplifying,
\begin{align}
\E[\REW]
    &\geq \textstyle\sum_{\myArms} \; N(a)\, r(a) + \OPTFD\cdot \nicefrac{B_0}{B}.
	\end{align}
Re-arranging, plugging in $B_0 = B - \sum_{a\neq a^*} N(a)\,c(a)$, and simplifying, we obtain
\begin{align*}
\OPTFD - \E[\REW]
    &\textstyle  \leq \sum_{\myArms} N(a) \left( \frac{\OPTFD}{B} \cdot c(a) - r(a) \right) & \\
	& \textstyle \leq \sum_{\myArms} N(a)\, \Dmin(a)
    & \EqComment{by \refeq{eq:gLagP}},
\end{align*}
and we are done by Claim~\ref{cl:FinalAlg}, just like in Case 1. 	
	
\xhdr{Case 3: $\nicefrac{B_0}{T_0} \leq c(a^*) \leq \nicefrac{B}{T}$.}
Let us write out $B_0$ and $T_0$:
\begin{align*}
c(a^*)
    & \geq \frac{B_0}{T_0}	
	= \frac{B - \sum_{\myArms} N(a)\, c(a)}{T - \sum_{\myArms} N(a)}
	 \geq \frac{B}{T} \rbr{ 1 - \frac{1}{B} \cdot \textstyle \sum_{\myArms} N(a) } \\
	& \geq \nicefrac{B}{T}  - O(\badDevSym/T),
        \;\text{where $\badDevSym$ is as in Theorem~\ref{thm:logRegretUpper}}
    & \EqComment{by Claim~\ref{cl:FinalAlg}}.
\end{align*}
	Since $c(a^*)\leq \nicefrac{B}{T}$, we have
    $0\leq \nicefrac{B}{T} - c(a^*) \leq O(\badDevSym/T)$
which contradicts the premise.
		
\xhdr{Case 4: $\nicefrac{B}{T} \leq c(a^*) \leq \nicefrac{B_0}{T_0}$.}
The argument is similar to Case 3. Writing out $B_0,T_0$, we have
\begin{align*}
c(a^*)
    & \leq \frac{B_0}{T_0}	
	= \frac{B - \sum_{\myArms} N(a) c(a)}{T - \sum_{\myArms} N(a)}
	\leq  \frac{B}{T(1 - \tfrac{1}{T} \cdot \sum_{\myArms} N(a))}.
\end{align*}
By Claim~\ref{cl:FinalAlg},
    $c(a^*) \leq \nicefrac{B}{T}\;(1+O(\badDevSym/T))$.
Therefore,
    $0\leq c(a^*) -\nicefrac{B}{T} \leq O(\badDevSym/T)$,
contradiction.

\section{Lower Bounds} \label{sec:LB}
We provide two lower bounds to complement Theorem~\ref{thm:logRegretUpper}: we argue that regret $\Omega(\sqrt{T})$ is essentially inevitable if a problem instance is far from best-arm-optimal or if there are $d>2$ resources.


We consider problem instances with three arms $\{A_1,A_2,\nullArm\}$, Bernoulli rewards, and $d\geq 2$ resources, one of which is time; call them \emph{$3\times d$ instances}. Each lower bound constructs two similar problem instances $\mI,\mI'$ such that any algorithm incurs high regret on at least one of them.%
\footnote{A standard approach for lower-bounding regret in multi-armed bandits is to construct multiple problem instances. A notable exception is the celebrated $\Omega(\log T)$ lower bound in \citet{Lai-Robbins-85}, which considers one (arbitrary) problem instance, but makes additional assumptions on the algorithm.}
The two instances have the same parameters $T,K,d,B$, and the mean reward and the mean consumption for each arm and each resource differ by at most $\eps$; we call them \emph{$\eps$-perturbation} of each other.


We start with an ``original" problem instance $\mI_0$ and construct problem instances $\mI,\mI'$ that are small perturbations of $\mI_0$. This is a fairly general result: unlike many bandit lower bounds that focus on a specific pair $\mI,\mI'$, we allow a wide range for $\mI_0$, as per the assumption below.


\newcommand{\cLB}{c_{\mathtt{LB}}}

\begin{assumption}\label{ass:LBAss}
There exists an absolute constant $\cLB\in (0,\nicefrac13)$ such that:
\begin{OneLiners}

\item[1.] \label{boundReward} $r(A_i),\,c_j(A_i)\in [\cLB,\,1-\cLB]$
for each arm $i\in \{1,2\}$ and each resource $j$.

\item[2.] \label{rewardDiff} $r(A_2) - r(A_1) \geq \cLB$ and $c_j(A_2) - c_j(A_1) \geq \cLB + \Dmin$ for every resource $j \in [d]$.

\item[3.] \label{boundOPT} $B \leq \cLB \cdot T \leq \OPTFD$.

\item[4.] \label{boundLagrange} Lagrangian gap is not extremely small:
$\Dmin \geq \cLB/\sqrt{T}$.

\end{OneLiners}
\end{assumption}

For a concrete example, let us construct a family of $3\times d$ problem instances that satisfy these assumptions. Fix some absolute constants
    $\eps,\cLB \in (0, \nicefrac13)$
and time horizon $T$. The problem instance is defined as follows:
budget $B= \cLB\,T$,
mean rewards
    $r(A_1) = \tfrac{1-\cLB}{2}$ and $r(A_2) = 1- \cLB - \eps$,
mean consumptions
    $c(A_1) = \cLB - \epsilon$ and $c(A_2) = 2 \cLB$.
Parts (1-4) of Assumption \ref{ass:LBAss} hold trivially. One can work out that
$\Dmin = \eps$, so part (4) holds as long as $\eps \geq \cLB/\sqrt{T}$.

\begin{theorem}\label{thm:LB-root}
Posit an arbitrary time horizon $T$, budget $B$, and $d$ resources (including time). Fix any $3\times d$  problem instance $\mI_0$ which satisfies Assumption~\ref{ass:LBAss}. In part (a), assume that
$d=2$ and $\mI_0$ is far from being best-arm-optimal, in the sense that
\begin{align}\label{eq:cor:bestArmOptimal}
\text{There exists an optimal solution $\vec{X}^*$ such that
$X(A_1)> 2 \cLB^4/\sqrt{T}$ and $X(A_2) \geq \cLB$}.
\end{align}
In part (b), assume that $d>2$.
For both parts, there exist problem instances $\mI,\mI'$, which are
$\mathcal{O}\rbr{\nicefrac{1}{\sqrt{T}}}$-perturbations of $\mI_0$, such that
\begin{align}\label{eq:LB-guarantee}
\text{Any algorithm incurs regret
    $\OPTFD-\E[\REW] \geq \Omega(\; \cLB^{4}\; \sqrt{T} \;)$
on $\mI$ or $\mI'$}
\end{align}
\end{theorem}

For part (a), instance $\mI$ has the same expected outcomes as $\mI_0$ (but possibly different outcome distributions); we call such problem instances \emph{mean-twins}. For part (b), one can take $\mI_0$ to be best-arm-optimal. For both parts, the problem instances $\mI,\mI'$ require randomized resource consumption.
	
Both parts follow from a more generic lower bound which focuses on linear independence of per-resource consumption vectors
 $\vec{c}_j  := \rbr{ c_j(A_1),\, c_j(A_2),\, c_j(\nullArm)} \in [0, 1]^3$,
 resources $j\in[d]$.

	\begin{theorem}\label{thm:generalLB}
		Posit an arbitrary time horizon $T$, budget $B$, and $d\geq 2$ resources (including time). Fix any $3\times d$ problem instance $\mI_0$ that satisfies Assumption~\ref{ass:LBAss} and \refeq{eq:cor:bestArmOptimal}. Assume that the consumption vectors $\vec{c}_j$, $j\in[d]$ are linearly independent. Then there are instances $\mI,\mI'$ which are $\eps$-perturbations of $\mI_0$, with
    $\eps = 2\,\cLB^2 / \sqrt{T}$,
which satisfy \eqref{eq:LB-guarantee}. In fact, $\mI$ is a mean-twin of $\mI_0$.
	\end{theorem}

\begin{myproof}[Sketch \textnormal{(see Appendix~\ref{sec:LB-generic} for full proof).}]
Let $r(a)$ and $\vc(a)\in [0,1]^d$ be, resp., the mean reward and the mean resource consumption vector for each arm $a$ for instance $\mI_0$. Let $\eps = \cLB/\sqrt{T}$.

Problem instances $\mI,\mI'$ are constructed as follows. For both instances, the rewards of each non-null arm $a\in \{A_1,A_2\}$ are deterministic and equal to $r(a)$. Resource consumption vector for arm $A_1$ is deterministic and equals $\vc(A_1)$. Resource consumption vector of arm $A_2$ in each round $t$, denoted $\vc_{(t)}(A_2)$, is a carefully constructed random vector whose expectation is $c(A_2)$ for instance $\mI$, and slightly less for instance $\mI'$. Specifically,
    $\vc_{(t)}(A_2) = \vc(A_2)\cdot W_t/(1-\cLB) $,
where $W_t$ is an independent Bernoulli random variable which correlates the consumption of all resources. We posit $\E[W_t] = 1-\cLB$ for instance $\mI$, and $\E[W_t] = 1-\cLB-\eps$ for instance $\mI'$.

Because of the small differences between $\mI,\mI'$, any algorithm will choose a sufficiently ``wrong" distribution over arms sufficiently often. The assumption in \refeq{eq:cor:bestArmOptimal} and the linear independence condition are needed to ensure that ``wrong" algorithm's choices result in large regret.
\end{myproof}

\newcommand{\NewI}{\widetilde{\mI}_0}

The corollaries are obtained as follows. For Theorem~\ref{thm:LB-root}(a), problem instance $\mI_0$ trivially satisfies all preconditions in Theorem~\ref{thm:generalLB}.
Indeed, letting time be resource $1$, the per-resource vectors are
    $\vec{c}_1 = (0,0,1)$
and
    $\vec{c}_2 = (\,\cdot\,,\,\cdot\,,\,0)$,
hence they are linearly independent. For Theorem~\ref{thm:LB-root}(b), we use some tricks from the literature to transform the original problem instance $\mI_0$ to another instance $\NewI$ which satisfies
\refeq{eq:cor:bestArmOptimal} and the linear independence condition. The full proof is in Section~\ref{sec:LB-cor}.

\OMIT{ 
\kaedit{The problem instances for Theorem~\ref{cor:multipleResources} is constructed as follows. Define
	\[	\zeta_1 := \min \left \{ \tfrac{1}{\sqrt{T}}, \{c_j(A_i)\}_{j \in [d], i \in [2]}, \frac{1}{(d!)^2} \right \}, \qquad and
	\]
	\[
		\zeta_2 := \min \left\{ \{c_j(A_i)\}_{j \in [d], i \in [2]}, \tfrac{1}{\sqrt{T}} \right \}.
	\]
	Given instance $\mI_0$, we construct instance $\mI$ by decreasing the mean consumption on arm $A_i$ and resource $j$ (except time) by $\zeta_1^j + u_j(A_i)$ where $u_j(a) \sim [-\zeta_2, \zeta_2]$ uniformly at random. We keep the mean rewards the same. As before $\mI'$ is obtained from $\mI$ by decreasing the mean consumption on all resources, except time, of one arm (say $A_2$) by $\eps = \mathcal{O}\left(\frac{1}{\sqrt{T}}\right)$ while keeping all other mean rewards/consumptions the same as in $\mI$.}
} 


%
%

\newpage
\section{Simple regret of \UCBBwK algorithm} \label{sec:simple-regret}
We define \emph{simple regret} in a given round $t$ as
    $\OPTDP/T- r(\vX_t)$,
where $\vX_t$ is the distribution over arms chosen by the algorithm. The benchmark $\OPTDP/T$ generalizes the best-arm benchmark from stochastic bandits. If each round corresponds to a user and the reward is this user's utility, then $\OPTDP/T$ is the ``fair share" of the total reward. We prove that with \UCBBwK, all but a few users receive close to their fair share. This holds if $B>\Omega(T) \gg K$, without any other assumptions.

\begin{theorem}\label{thm:UCBSmallNonArms}
Consider \UCBBwK. Assume $B \geq \Omega(T)$ and $\myEta\leq \tfrac12$. With probability $\geq 1-O(T^{-3})$, for each $\eps>0$, there are at most
$N_\eps = \mathcal{O}\left( \frac{K}{\eps^2} \log KTd  \right)$ rounds $t$ such that
    $\OPTDP/T - r(\vX_t) \geq \eps$.
\end{theorem}


To prove Theorem~\ref{thm:UCBSmallNonArms}, we consider another generalization of the ``reward-gap", which measures the difference in LP-value compared to $\OPTLP$. For distribution $\vX$ over arms, the \emph{LP-gap} of $\vX$ is
\begin{align}\label{eq:LPgap-defn}
\Gap(\vX) := \OPTLP - \Val(\vX), \;\text{where}\;
V(\vX) := \textstyle (\nicefrac{B}{T})\;\cdot r(\vX)  / \rbr{\max_{j\in [d]} c_j(\vX)}.
\end{align}
Here, $V(\vX)$ is the value of $\vX$ in the LP~\eqref{lp:primalAbstract} after rescaling, so that $\OPTLP = \sup_{\vX} V(\vX)$. Note that $\vX$ does not need to be feasible for \eqref{lp:primalAbstract}.
It suffices to study the LP-gap because
    $r(\vX_t)\geq \Val(\vX_t) (1-\myEta)$
for each round $t$ with high probability. This holds under the ``clean event" in \eqref{eq:cleanEvent}, because $\vX_t$ being the solution to the optimistic LP implies
    $\max_j c_j(\vX_t) \geq \nicefrac{B}{T}\; (1-\myEta)$.

Thus, we upper-bound the number of rounds $t$ in which $\Gap(\vX_t)$ is large. We do this in two steps, focusing on the confidence radius $\rad_t(\vX_t)$ as defined in \eqref{eq:MaxConfRadUB}. First, we upper-bound the number of rounds $t$ with large $\rad_t(\vX_t)$. A crucial argument concerns \emph{confidence sums}:
\begin{align}\label{eq:confSums}
\textstyle
\sum_{t \in S}\; \rad_t(a_t)
\quad\text{and}\quad
\sum_{t \in S}\; \rad_t(\vX_t),
\end{align}
the sums of confidence radii over a given subset of rounds $S\subset [T]$, for, resp., actions $a_t$ and distributions $\vX_t$ chosen by the algorithm. Second, we upper-bound $\Gap(\vX_t)$ in terms of $\rad_t(\vX_t)$. 

\subsection{Confidence sums}
\label{sec:logRegret-confSum}


The following arguments depend only on the definition of the confidence radius, and work for any algorithm $\ALG$. Suppose in each round $t$, this algorithm chooses a distribution $\vY_t$ over arms and samples arm $a_t$ independently $\vY_t$. We upper-bound the number of rounds $t$ with large $\rad_t(\vY_t)$:

\begin{lemma}\label{lm:DistrConfSum-main}
Fix the threshold $\theta_0> 0$,
and let $S$ be the set of all rounds $t\in [T]$ such that
    $\rad_t(\vY_t)\geq \theta_0$.
Then
    $|S| \leq \mO\left( \theta_0^{-2}\cdot K \logThm \right)$
with probability at least $1-O(T^{-3})$.
\end{lemma}

To prove the lemma, we study \emph{confidence sums}:
for a subset $S \subset [T]$ of rounds, define
\begin{align*}
\AcConfSum(S)
    &:= \textstyle\sum_{t \in S}\; \rad_t(a_t)
    & \EqComment{action-confidence sum of \ALG}, \\
\DisConfSum(S)
    &:=  \textstyle\sum_{t \in S}\; \rad_t(\vY_t)
    & \EqComment{distribution-confidence sum of \ALG}.
\end{align*}

First, a standard argument (\eg implicit in \mycite{bandits-ucb1}, see Section~\ref{app:confSum-standard}) implies that
\begin{align}\label{eq:ActConfSum-UB}
\AcConfSum(S) \leq \mO \rbr{ \sqrt{K\, |S|\, \myGamma} + K \cdot \ln |S| \cdot \myGamma}
\quad\text{for any fixed subset $S \subset [T]$}.
\end{align}



Second, note that $\DisConfSum(S)$ is close to $\AcConfSum(S)$: for any fixed subset $S \subset [T]$,
\begin{align}\label{eq:twoConfSums}
\left| \DisConfSum(S) - \AcConfSum(S) \right| \leq
    \mathcal{O}(\sqrt{|S|\,\log T})
    \quad\text{with probability at least $1- T^{-3}$}.
\end{align}
This is by Azuma-Hoeffding inequality, since
    $\rbr{\rad_t(a_t) - \rad_t(\vY_t):\; t\in S}$
is a martingale difference sequence. We extend this observation to \emph{random} sets $S$. A random set $S\subset [T]$ is called \emph{time-consistent} if the event $\{t\in S\}$ does not depend on the choice of arm $a_t$ or anything that happens afterwards, for each round $t$. (But it \emph{can} depend on the choice of distribution $\vY_t$.)

\begin{claim}\label{cl:acConfDisConf}
For any any time-consistent random set $S\subset [T]$,
\begin{align}\label{eq:cl:acConfDisConf}
\left| \DisConfSum(S) - \AcConfSum(S) \right| \leq
    \mathcal{O}\left( \sqrt{|S|\,\log T} + \log T\right)
    \quad\text{with probability at least $1- T^{-3}$}.
\end{align}
\end{claim}

\begin{proof}
By definition of time-consistent set, for each round $t$,
\[ \E[\indicator{t\in S}\cdot\rad_t(a_t)
     \mid (\vY_1, a_1) \LDOTS (\vY_{t-1}, a_{t-1}),  \vY_t]
    = \indicator{t\in S}\cdot  \rad_t(\vY_t).\]
Thus, $\indicator{t\in S}\rad_t(a_t) - \rad_t(\vY_t)$, $t\in[T]$ is martingale difference sequence. Claim~\ref{cl:acConfDisConf} follows from a concentration bound from prior work (Theorem~\ref{thm:myAzuma}).
\end{proof}

We complete the proof of Lemma~\ref{lm:DistrConfSum-main} as follows.
Fix $\delta>0$. Since $S$ is a time-consistent random subset of $[T]$, by \refeq{eq:ActConfSum-UB} and Claim~\ref{cl:acConfDisConf}, with probability at least $1- \delta$ it holds that
\[
    \theta_0\cdot |S|
    \leq \DisConfSum(S)
    \leq \mathcal{O}
        \left( \sqrt{|S| K \myGamma} + K\,\myGamma
            + \sqrt{|S|\,\log T} + \log T
        \right).
\]
We obtain the Lemma by simplifying and solving this inequality for $|S|$.

\subsection{Connecting LP-gap and the confidence radius}
\label{sec:simpleRegret-confRad}


In what follows,
let $\B = B(1-\myEta)$ be the budget in the rescaled LP.

\begin{lemma} \label{lem:MainLemmaUCBBwK}
Fix round $t \in [T]$, and assume the ``clean event" in \eqref{eq:cleanEvent}. Then
			\[
					\Gap(\vX_t) \leq \left( 2 + \nicefrac{T}{\B} \right) \rad_t(\vX_t).
			\]	
	\end{lemma}

	\begin{proof}
		Let $\alpha := \B/T$. For any distribution $\vX$, let
    \[ \ValP(\vX) := \nicefrac{\B}{T}\;\cdot r(\vX)  / \max_{j\in [d]} c^-_j(\vX).\]
denote the value of $\vX$ in the optimistic $\LP$~\eqref{lp:UCBBwK}, after proper rescaling. Let $\vX^*$ be an optimal solution to the (original) LP ~\eqref{lp:primalAbstract}. Then
    \begin{align}
			\label{eq:LPGap}
				\Gap(\vX_t) = \Val(\vX^*) - \Val(\vX_t) - \ValP(\vX_t) + \ValP(\vX_t).
		\end{align}
		Since $\ValP(\vX_t)$ is the optimal solution to the optimistic $\LP$~\eqref{lp:UCBBwK},
		\[
			 \ValP(\vX_t) \geq \ValP(\vX^*).
		\]
		Moreover, since $\vX^*$ is feasible to the optimistic $\LP$~\eqref{lp:UCBBwK} with the scaled budget $\B$,
		\[
				\ValP(\vX^*) \geq \Val(\vX^*).
		\]
It follows that \refeq{eq:LPGap} an be upper-bounded as
		\begin{align}\label{eq:LPGap-1}
				\Gap(\vX_t) \leq \ValP(\vX_t) - \Val(\vX_t).
		\end{align}
We will now upper-bound the right-hand side in the above. Denote
\begin{align*}
c_{\max}(\vX_t)
    &:= \max_{j \in [d]} \sum_{a \in [K]} c_{j, t}(a)  X_t(a)\\
c^{-}_{\max}(\vX_t)
    &:= \max_{j \in [d]} \sum_{a \in [K]} c^{-}_{j, t}(a)  X_t(a).
\end{align*}
By definition of the value of a linear program, we can continue \refeq{eq:LPGap-1} as follows:
\begin{align}
\Gap(\vX_t)
    &\leq \ValP(\vX_t) - \Val(\vX_t) \nonumber \\
	&\leq \alpha\cdot \frac{\hat{r}(\vX_t) + \rad_t(\vX_t)}{c^{-}_{\max}(\vX_t)} - \alpha \cdot \frac{r(\vX_t)}{c_{\max}(\vX_t)}. \label{eq:defnGap}
\end{align}
Under the clean event in \refeq{eq:cleanEvent}, we continue \refeq{eq:defnGap} as follows:
		\begin{align}
			& \leq \alpha \left( \frac{2 \rad_t(\vX_t) + r(\vX_t) }{ c^{-}_{\max}(\vX_t) } - \frac{r(\vX_t)}{c_{\max}(\vX_t)} \right). \label{eq:clEventLPGapUsage}
		\end{align}
Since time is one of the resources,
    $c^{-}_{\max}(\vX_t) \geq \frac{\B}{T}$.
Thus,  we continue \refeq{eq:clEventLPGapUsage} as follows:
		\begin{align}
			& 	\leq 2 \rad_t(\vX_t)  + \alpha r(\vX_t) \left( \frac{1}{c^{-}_{\max}(\vX_t)} - \frac{1}{c_{\max}(\vX_t)} \right) \nonumber \\
			& = 2 \rad_t(\vX_t)  + \alpha r(\vX_t) \left( \frac{\rad_t(\vX_t)}{c^{-}_{\max}(\vX_t) \cdot c_{\max}(\vX_t)} \right) \nonumber \\
			& \leq 2 \rad_t(\vX_t)  + \frac{\rad_t(\vX_t)}{c^{-}_{\max}(\vX_t)} \label{eq:clEventLPGapUsage2} \\
			& \leq \left( 2 + \tfrac{T}{\B} \right) \rad_t(\vX_t) \label{eq:clEventLPGapUsage3}
		\end{align}
		\refeq{eq:clEventLPGapUsage2} uses the fact that $\alpha \frac{r(\vX_t)}{c_{\max}(\vX_t)} \leq \frac{B}{T} \frac{r(\vX_t)}{c_{\max}(\vX_t)} = V(\vX_t) \leq 1$. \refeq{eq:clEventLPGapUsage3} uses the fact that time is one of the resources and thus, $c^{-}_{\max}(\vX_t) \geq \frac{\B}{T}$.
	\end{proof}

%
%
%

\subsection{Finishing the proof of Theorem~\ref{thm:UCBSmallNonArms}}

\begin{claim}\label{cl:simpleRegret-to-Gap}
Fix round $t$, and assume the ``clean event" in \eqref{eq:cleanEvent}. Then
\[ \OPTDP/T - r(\vX_t) \leq \Gap(\vX_t) + \myEta.\]
\end{claim}
\begin{proof}
By \eqref{eq:cleanEvent} and because $\vX_t$ is the solution to the optimistic LP, we have
   \[ \max_{j\in d} c_j(\vX_t)
        \geq \max_{j\in d} c^-_j(\vX_t) = \nicefrac{B}{T}\; (1-\myEta).\]
It follows that $r(\vX_t) \geq V(\vX_t)(1-\myEta)$.
 Finally, we know that
    $\OPTLP\geq \OPTDP/T$.
\end{proof}

Condition on \eqref{eq:cleanEvent}, and the high-probability event in Lemma~\ref{lm:DistrConfSum-main}. (Take the union bound in Lemma~\ref{lm:DistrConfSum-main} over all thresholds $\theta_0\geq 1/\sqrt{T}$, \eg over an exponential scale.) Fix $\eps>0$. By Claim~\ref{cl:simpleRegret-to-Gap} and Lemma~\ref{lem:MainLemmaUCBBwK}, any round $t$ with simple regret at least $\eps$ satisfies
\[ \eps \leq \OPTDP/T-r(\vX_t)  \leq
    \myEta + \left( 2 + \nicefrac{T}{\B} \right) \rad_t(\vX_t).\]
Therefore,
    $\rad_t(\vX_t) \geq \theta_0$, where
 $\theta_0 =  \frac{\epsilon - \myEta}{ \left( 2 + \nicefrac{T}{\B} \right)} \geq \Theta(\epsilon)$ when $\epsilon \geq 2 \myEta$.
Now, the theorem follows from Lemma~\ref{lm:DistrConfSum-main}. Note, when $\epsilon < 2 \myEta$, then the total number of rounds in the theorem is larger than $T$ and hence not meaningful.	

\subsection{The standard confidence-sum bound: proof of \refeq{eq:ActConfSum-UB}}
\label{app:confSum-standard}

Let us prove \refeq{eq:ActConfSum-UB} for the sake of completeness. By definition of $\rad_t(a_t)$ from \refeq{eq:MaxConfRadUB},
    \[ \rad_t(a_t) = f(n) := \min\rbr{ 1,\;\sqrt{\myGamma/n} + \myGamma/n }, \]
where $N_t(a)$ is the number of times arm $a$ was chosen before round $t$. Therefore:
\begin{align*}
\sum_{t \in S} \rad_t(a_t)
    &\leq \sum_{a \in [K]} \sum_{n=1}^{|S|/K} f(n) \\
    &\leq \sum_{a \in [K]}  \int_{x=1}^{|S|/K} f(x)\, \mathrm{d}  x
    \leq 3 \rbr{ \sqrt{ K |S|\, \myGamma} + K\cdot\ln |S| \cdot \myGamma }.
\end{align*}

\section{Reduction from \BwK to stochastic bandits}
 \label{sec:extensions}
We improve all regret bounds for \UCBBwK algorithm, from worst-case regret to logarithmic regret to simple regret, when the problem instance has some helpful structure. In fact, we provide a general \emph{reduction} which translates insights from stochastic bandits into results on \BwK. This reduction works as follows: if prior work on a particular scenario in stochastic bandits provides an improved upper bound on the confidence sums \eqref{eq:confSums},
this improvement propagates throughout the analyses of \UCBBwK. Specifically, suppose
    $\sum_{t\in S} \rad_t(a_t) \leq \sqrt{\beta\, |S|}$
for all algorithms, all subsets of rounds $S\subset [T]$, and some instance-dependent parameter $\beta\ll K$,
then \UCBBwK satisfies
\begin{OneLiners}
\item[(i)] worst-case regret
  $\OPTDP -\E[\REW] \leq O(\sqrt{\beta T})(1+\OPTDP/B)$.

\item[(ii)]
Theorem~\ref{thm:logRegretUpper} holds with
    $\Psi = \beta\,\Dmin^{-2}$
and regret
    $\mO\rbr{ \beta\,\Dmin^{-1} }$
in part (ii).

\item[(iii)] Theorem~\ref{thm:UCBSmallNonArms} holds with
$N_\eps = \mO\left( \beta\,\eps^{-2}  \right)$.
\end{OneLiners}			
Conceptually, this works because confidence sum arguments depend only on the confidence radii, rather than the algorithm that chooses arms, and are about stochastic bandits rather than \BwK. The analyses of \UCBBwK in \citep{AgrawalDevanur-ec14} and the previous sections use $\beta=K$, the number of arms. The confidence sum bound with $\beta=K$ and results (i, ii, iii) for stochastic bandits follow from the analysis in \cite{bandits-ucb1}.


We apply this reduction to three well-studied scenarios in stochastic bandits:
combinatorial semi-bandits
\citep[\eg][]{Chen-icml13,Kveton-aistats15,MatroidBandits-uai14},
linear contextual bandits
\citep[\eg][]{Auer-focs00,DaniHK-colt08,Langford-www10,Reyzin-aistats11-linear,Csaba-nips11},
and multinomial-logit (MNL) bandits \citep[\eg][]{Shipra-ec16,rusmevichientong2010dynamic,saure2013optimal,caro2007dynamic}.
The confidence-sum bounds are implicit in prior work on stochastic bandits, and we immediately obtain the corresponding extensions for \BwK. To put this in perspective, each scenario has lead to a separate paper on  \BwK \citep[resp.,][]{Karthik-aistats18,agrawal2015linear,Cheung-MNLBwK-arxiv17}, for the worst-case regret bounds alone. We essentially  match the worst-case regret bounds from prior work, and obtain new bounds on logarithmic regret and simple regret.%
\footnote{However, we do not provide a generic computationally efficient implementation.} 

\subsection{The general result}
We extend our results to any problem which can be cast as a special case of \BwK and admits an upper bound on action-confidence sums, in the style of \eqref{eq:ActConfSum-UB}, for a suitably defined confidence radius.



To state the general result, let us define an abstract notion of ``confidence radius". For each round $t$, a \emph{formal confidence radius} is a mapping $\rad_t(a)$ from algorithm's history and arm $a$ to $[0,1]$ such that with probability at least $1-O(T^{-4})$ it holds that
\[
    |r(a)-\hat{r}_t(a)| \leq \rad_t(a)
\quad\text{and}\quad
    |c_j(a)-\hat{c}_{j,t}(a)| \leq \rad_t(a)
\]
for each resource $j$, where $\hat{r}_t(a)$ and $\hat{c}_{j, t}(a)$ denote average reward and resource consumption, as defined in \refeq{eq:ave-defn}.
 Such $\rad_t(a)$ induces a version of $\UCBBwK$ with confidence bounds
\[ r_t^{+}(a) = \min(1, \hat{r}_t(a) + \rad_t(a) \;)
\quad\text{and}\quad
c_{j, t}^{-}(a) = \max(\;0, \hat{c}_{j, t}(a) - \rad_t(a) \;).\]

We allow the algorithm to observe auxiliary feedback before and/or after each round, depending on a particular problem formulation, and this feedback may be used to compute the confidence radii.

We replace \refeq{eq:ActConfSum-UB} with a generic bound on the action-confidence sum, for some $\beta$ that can depend on the parameters in the problem instance, but not on $S$:
\begin{align}\label{eq:ConfSumBound-generic}
\textstyle \sum_{t\in S} \rad_t(a_t) \leq \sqrt{|S|\, \beta},
\quad\text{for any algorithm and any subset $S\subset[T]$}.
\end{align}

\begin{theorem}\label{thm:abstractReduceMain}
Consider an instance of $\BwK$ with time horizon $T$. Let $\rad_t(\cdot)$ be a formal confidence radius which satisfies \eqref{eq:ConfSumBound-generic} for some $\beta$.
 Consider the induced algorithms $\UCBBwK$ and $\FinalAlgName$ with rescaling parameter $\myEta = \frac{2}{B} \sqrt{\beta T}$.
\begin{OneLiners}
\item[(i)] Both algorithms obtain regret
  $\OPTDP -\E[\REW] \leq O(\sqrt{\beta T})(1+\OPTDP/B)$.

\item[(ii)]
Theorem~\ref{thm:logRegretUpper} holds with
    $\Psi = \beta\,\Dmin^{-2}$
and regret
    $\mO\rbr{ \beta\,\Dmin^{-1} }$
in part (ii).

\item[(iii)] Theorem~\ref{thm:UCBSmallNonArms} holds with
$N_\eps = \mO\left( \beta\,\eps^{-2}  \right)$.
\end{OneLiners}			
	\end{theorem}

\begin{myproof}[Sketch]
For part (i), the analysis in \mycite{AgrawalDevanur-ec14} explicitly relies on \eqref{eq:ActConfSum-UB}. For part (ii), we modify the proof of Theorem~\ref{thm:logRegretUpper} so as to use \eqref{eq:ActConfSum-UB} instead of Claim~\ref{cl:ConfSum-main}. For part (iii), our proof of Theorem~\ref{thm:UCBSmallNonArms} uses \eqref{eq:ActConfSum-UB} explicitly. In all three parts, we replace \eqref{eq:ActConfSum-UB} with \eqref{eq:ConfSumBound-generic}, and trace how the latter propagates through the respective proof.
\end{myproof}

We apply this general result to three specific scenarios:
linear contextual bandits with knapsacks (\LinCBwK)
\citep{agrawal2015linear},
combinatorial semi-bandits with knapsacks (\SemiBwK)
\citep{Karthik-aistats18},
and multinomial-logit bandits with knapsacks (\MnlBwK)
\citep{Cheung-MNLBwK-arxiv17}.
In all three applications, the confidence-sum bound \eqref{eq:ConfSumBound-generic} is implicit in prior work on the respective  problem without resources.
The guarantees in part (i) match those in prior work referenced above, up to logarithmic factors, and are optimal when $B = \Omega(T)$; in fact, we obtain an improvement for \MnlBwK.
Parts (ii) and (iii) -- the results for logarithmic regret and simple regret -- did not appear in prior work.

\subsection{Linear Contextual Bandits with Knapsacks (\LinCBwK)}

In \emph{Contextual Bandits with Knapsacks} (\CBwK), we have $K$ actions, $d$ resources, budget $B$ and time horizon $T$, like in \BwK, and moreover we have a set $\mX$ of possible contexts. At each round $t \in [T]$, the algorithm first obtains a context $\vx_t\in X$. The algorithm then chooses an action $a_t \in [K]$ and obtains an outcome
    $\vo_t(a_t) \in [0, 1]^{d+1}$
like in \BwK.
The tuple
    $\rbr{ \vx_t; \vo_t(a):\, a \in [K] }$
is drawn independently from some fixed but unknown distribution. The algorithm continues until some resource, including time, is exhausted. One compares against a given a set $\Pi$ of \emph{policies}: mappings from contexts to actions. We can formally interpret \CBwK as an instance of \BwK in which actions correspond to policies in $\Pi$. This interpretation defines the benchmarks $\OPTDP$ and $\OPTFD$ that we compete with.

\LinCBwK is a special case of \CBwK in which the context space is $\mX = [0,1]^{K\times m}$, for some parameter $m\in\N$, so that each context $\vx_t$ is in fact a tuple
    $\vx_t = \rbr{\vx_t(a)\in [0,1]^m:\,a\in[K] }$.
We have a linearity assumption: for some unknown matrix $\vec{W}_* \in [0, 1]^{m \times (d+1)}$ and each arm $a \in [K]$,
\[ \E\sbr{ \vo_t(a) \mid \vx_t(a) } = \vec{W}_*^\textrm{T} \cdot \vx_t(a). \]
The policy set $\Pi$ consists of all possible policies.


	

\emph{Linear contextual bandits}, studied in prior work
\citep[\eg][]{Auer-focs00,DaniHK-colt08,Langford-www10,Reyzin-aistats11-linear,Csaba-nips11},
is the special case without resources. Much of the complexity of linear contextual bandits (resp., \LinCBwK) is captured by the special case of of \emph{linear bandits} (resp., \emph{linear \BwK}) where the context is the same in each round.

The general theme in the work on linear bandits (contextual or not) to replace the dependence on the number of arms $K$ in the regret bound with the dependence on the dimension $m$ and, if applicable, avoid the dependence on $|\Pi|$. This is what we accomplish, too.

\begin{corollary}\label{cor:linCBwK}
For \LinCBwK,
Theorem~\ref{thm:abstractReduceMain} holds with $\beta = \mO(m^2 d^2 \log(mTd))$.
\end{corollary}
	
\begin{proof}
Combining Lemma 13 of \mycite{Auer-focs00} and Theorem 2 of \mycite{AbbasiYPS-nips11}, it follows that the confidence-sum bound \refeq{eq:ConfSumBound-generic} holds with
$\beta = \mO(m^2 d^2 \log mTd)$.
\end{proof}

\subsection{Combinatorial Semi-bandits with Knapsacks (\SemiBwK)}

\SemiBwK is a version of \BwK, where actions correspond to subsets of some fixed ground set $[N]$ (whose elements are called \emph{atoms}). There is a fixed family $\mF \subset 2^{[N]}$ of feasible actions. In each round $t$, the algorithm chooses a subset $A_t \in \mF$ and observes the outcome
    $\vo_t(a)\in [0,\nicefrac{1}{n}]^d$
for each atom $a\in A_t$, where
    $n = \max_{A \in \mF} |A|$.
The outcome for a given subset $A\in \mF$ is defined as the sum
\begin{align}\label{eq:SemiBwK-sum}
    \vo_t(A) = \textstyle \sum_{a\in A} \vo_t(a) \in [0,1]^{d+1}.
\end{align}
The outcome matrix $\rbr{\vo_t(a): a \in [N]}$ is drawn independently from some fixed but unknown distribution. The algorithm continues until some resource, including time, is exhausted.

\emph{Combinatorial semi-bandits}, the problem studied in prior work
\citep[\eg][]{Chen-icml13,Kveton-aistats15,MatroidBandits-uai14},
is the special case without resources. Note that the number of feasible actions can be exponential in $N$. The general theme in this line of work is to replace the dependence on $|\mF|$ in the regret bound with the dependence on $N$, or, even better, on $n$. We extend this to \SemiBwK.

	
	
\begin{corollary}\label{cor:semiBwK}
For \SemiBwK, Theorem~\ref{thm:abstractReduceMain} holds with
$\beta = \mO(n \log(NdT))$.
\end{corollary}

\begin{proof}
Using Lemma 4 in \mycite{wen2015efficient} we immediately obtain the confidence-sum bound \refeq{eq:ConfSumBound-generic} with $\beta = n \log KdT$.
\end{proof}

\subsection{Multinomial-logit Bandits with Knapsacks (\MnlBwK)}

In the \MnlBwK problem, the setup starts like in \SemiBwK. There is a ground set of $N$ \emph{atoms}, and a fixed family $\mF \subset 2^{[N]}$ of feasible actions. In each round, each atom $a$ has an outcome $\vo_t(a)\in [0,1]^{d+1}$, and the outcome matrix $\rbr{\vo_t(a): a \in [N]}$ is drawn independently from some fixed but unknown distribution. The aggregate outcome is formed in a different way: when a given subset $A_t\in\mF$ is chosen by the algorithm in a given round $t$, at most one atom $a_t\in A_t$ is chosen stochastically by ``nature", and the aggregate outcome is then $\vo_t(A_t) := \vo_t(a)$; otherwise, the algorithm skips this round. A common interpretation is that the atoms correspond to products, the chosen action $A_t\in \mF$ is the bundle of products offered to the customer, and at most one product from this bundle is actually purchased. As usual, the algorithm continues until some resource (incl. time) is exhausted.


The selection probabilities are defined via the multinomial-logit model. For each atom $a$ there is a hidden number $v_a\in [0,1]$, interpreted as the customers' valuation of the respective product, and the
\[\Pr\sbr{ \text{atom $a$ is chosen} \mid A_t} =
\begin{cases}
	\tfrac{v_a}{1+\sum_{a' \in A_t} v_{a'}} & \text{if $a \in A_t$} \\
			0 & \text{otherwise}.
\end{cases}
\]
The set $\mF$ of possible bundles is
    \[ \mF = \cbr{ A\subset [N]:\; \vec{M} \cdot x(A) \leq \vec{b} },\]
for some (known) totally unimodular matrix $\vec{M}\in \R^{N\times N} $ and a vector $\vec{b} \in \R^N$, where $x(A)\in \{0,1\}^N$ represents set $A$ as a binary vector over atoms.
	
\emph{Multinomial-logit bandits}, the problem studied in prior work \citep[\eg][]{Shipra-ec16,rusmevichientong2010dynamic,saure2013optimal,caro2007dynamic},
is the special case without resources. We derive the following corollary from the analysis of MNL-bandits in  \citet{Shipra-ec16}, which analyzes the confidence sum for the $v_a$'s.

\begin{corollary}\label{cor:MnlBwK}
Consider \MnlBwK and denote
    $V := \sum_{a \in [N]} v_a$.
 Theorem~\ref{thm:abstractReduceMain} holds with
\[ \beta = \textstyle
    \mO\rbr{\rbr{\frac{\ln T}{\ln (1 + \nicefrac{1}{V})}}^2
        \rbr{N \sqrt{\ln (NT)} + \ln(NT) }} = \tildeO\rbr{N^3}.\]
\end{corollary}

\begin{proof}
The proof is implicit in the analysis in \citet{Shipra-ec16}. As in their paper, let $n_\ell$ denote the number of time-steps in phase $\ell$. Let $V_\ell = \sum_{a \in S_\ell} v_a$. Recall that $n_\ell$ is a geometric random variable with mean $\tfrac{1}{1 + V_\ell}$. Using Chernoff-Hoeffding bounds we obtain that with probability at least $1 - \tfrac{1}{T^2}$, $n_\ell \leq \tfrac{\ln T}{\ln (1+\nicefrac{1}{V_\ell})}$.

	Consider a random subset $S$. Summing the LHS and RHS in Lemma 4.3, we get that $\sum_{t \in S} \rad_t(a_t) \leq \sum_{a \in [N]} \sum_{\ell: t \in \mathcal{T}_a(\ell)} \tilde{R}_a(S_\ell)$.
	Using Lemma 4.3 in \citep{Shipra-ec16} we have, $\sum_{a \in [N]} \sum_{\ell: t \in \mathcal{T}_a(\ell)} \tilde{R}_a(S_\ell) \leq \sum_{a \in [N]} \sum_{\ell: t \in \mathcal{T}_a(\ell)} n_{\ell} \sqrt{\frac{v_a \ln \sqrt{N} T}{T_a(\ell)}} + \frac{\ln \sqrt{N} T}{T_a(\ell)}$. Note that $v_a \leq 1$. Using the upper bound on $n_\ell$ derived above combined with the argument used to obtain (A.19) in \citep{Shipra-ec16} we get the desired value of $\beta$.
\end{proof}

The worst-case regret bound from Corollary~\ref{cor:MnlBwK} improves over prior work \mycite{Cheung-MNLBwK-arxiv17}. In particular, consider the worst-case dependence on $N$, the number of atoms. Our regret bound scales as $N^{3/2}$, whereas the regret bound in \citep{Cheung-MNLBwK-arxiv17} scales as $N^{7/2}$ (while both scale as $\sqrt{T}$).

\subsection{Computational issues}

We do not provide a generic computationally efficient implementation for \UCBBwK in our reduction. The algorithm constructs and solves a linear program in each round, with one variable per arm in the reduction. So, even if the regret is fairly small, the number of LP variables may be very large: indeed, it may be exponential in the number of atoms in \SemiBwK and \MnlBwK, arbitrarily large compared to the other parameters in linear \BwK, or even infinite as in \LinCBwK. The corresponding LPs have a succinct representation in all these applications, but we do not provide a generic implementation. However, such (or very similar) linear programs may be computationally tractable via application-specific implementations, and indeed this is the case in \LinCBwK \citep{agrawal2015linear} and \SemiBwK \citep{Karthik-aistats18}. In the prior work on \MnlBwK \citep{Cheung-MNLBwK-arxiv17}, the $\sqrt{T}$-regret algorithm is not computationally efficient, same as ours; there is, however, a computationally efficient algorithm with regret $T^{2/3}$. 

\section{Discussion: significance and novelty} 
Characterizing (poly-)logarithmic regret rates is a very natural question, and we give a complete answer. The answer consists of positive and negative parts: the positive part requires substantial assumptions, and these assumptions are necessary. The positive result comes ``for free" despite the assumptions: it is achieved via \UCBBwK and  without sacrificing the worst-case performance.

The $O(\log T)$ regret result is well-motivated on its own, even though it requires $d=2$ and best-arm-optimality and a reasonably small $K=\text{\#arms}$. Indeed, problems with $d=2$ and small $K$ arise in many motivating applications of \BwK (see Appendix~\ref{app:examples}), and capture the three challenges of \BwK discussed in the Introduction. Moreover, best-arm-optimality is a typical, non-degenerate case.
\footnote{To make this point formal, we focus on $d=2$ and observe that best-arm-optimality arises with probability at least $p$, for some absolute constant $p>0$, if expected rewards and expected resource consumptions are drawn independently and uniformly at random. This is a generic fact about LPs, which follows, \eg from the definition of primal degeneracy in Section 2 of \mycite{megiddo1988perturbation}, combined with Proposition 2.7.2 in \mycite{terryTaoBook}.}


For lower bounds in terms of Lagrangian gap $\Dmin$, we rely on the $\Omega(\nicefrac{1}{G} \cdot \log T)$ regret bound for bandits \citep{Lai-Robbins-85}, where $G$ is the reward-gap (since $\Dmin$ generalizes reward-gap). In particular, $1/\Dmin$ scaling is optimal. No other instance-dependent lower bounds are known for \BwK. However, Theorem~\ref{thm:LB-root} implies $\Omega(\sqrt{T})$ regret for some ``proper" instances of \BwK (\ie ones with resource consumption) that have small $\Dmin$.

Simple regret is a standard performance measure in stochastic bandits, previously not studied for  \BwK. While our result requires
    $B>\Omega(T) \gg K$,
this is the main ``parameter regime" of interest in most/all prior work on \BwK, and a necessity in an important subset of this work \citep{BZ09,BesbesZeevi-or12,Wang-OR14,AdvBwK-focs19}. In contrast with stochastic bandits, Theorem~\ref{thm:UCBSmallNonArms} does not imply logarithmic regret, as per our lower bounds.

The ``reduction" result is conceptual rather than technical. We make the point that regret bounds for many extensions of \BwK can be derived seamlessly, and identify a mathematical structure which drives these extensions (namely, a bound on confidence sums). In a way, we formalize the intuition that analyses of ``optimism under uncertainty" are likely to carry over from stochastic bandits to \BwK.

We introduce several new concepts and techniques: \emph{Lagrangian gap} \eqref{eq:Glag} for logarithmic regret, \emph{LP-gap} \eqref{eq:LPgap-defn} for analyzing simple regret, and the abstraction of \emph{confidence sums} \eqref{eq:confSums}. Also, LP-sensitivity arguments appear new in bandit analyses. Both new notions of ``gap" satisfy the natural desiderata: they generalize reward-gap, separate the dependence on the problem instance from that on the time horizon $T$ (formally: do not depend on $T$, fixing the $\nicefrac{B}{T}$ ratio), and are ``productive", leading to improved results. However, neither notion captures \emph{all} \BwK instances with low regret.%
\footnote{This should not be surprising per se, as reward-gap does not capture all ``nice" bandit instances either. \Eg problem instances with small reward-gap admit $O(\log T)$ regret if they have a likewise small best reward.}

\clearpage
\begin{small}
\bibliographystyle{plainnat}
\bibliography{bib-abbrv,bib-AGT,bib-bandits,bib-ML,bib-random,bib-slivkins,refs}
\end{small}



\clearpage
\begin{appendices}

\tableofcontents

\addtocontents{toc}{\protect\setcounter{tocdepth}{2}}

\section{Motivating examples with $d=2$ and small $K$}
\label{app:examples}
We provide direct motivation for Theorem~\ref{thm:logRegretUpper}, our positive result for $O(\log T)$ regret. Recall that Theorem~\ref{thm:logRegretUpper} only holds with $d=2$ resources, and is only meaningful with a reasonably small number of arms $K$ (because the regret bounds are linear in $K$). Such problems arise in many motivating applications of \BwK, \eg as listed in \cite{BwK-focs13-conf,BwK-focs13}. Below we spell out several stylized examples.


In \emph{dynamic assortment} \citep{Zeevi-assortment-13,Shipra-ec16,Cheung-MNLBwK-arxiv17}, an algorithm is a seller which chooses among possible assortments of products. In each round, a customer arrives, the algorithm chooses an assortment, and offers this assortment for sale at an exogenously fixed price. If a sale happens, the algorithm receives revenue and consumes some amount of inventory. The following version features $d=2$ and non-huge $K$: there are $K$ possible offerings for sale, and a limited amount of ``raw material" used to manufacture them. Each offering, if sold, consumes some pre-fixed amount of this raw material.%
\footnote{This framing with raw material(s) --- \BwK formulations of revenue management problems in which products being sold are separate from raw material(s) being consumed --- traces back to \citet{BesbesZeevi-or12}.}

The ``inverted" dynamic assortment problem takes the procurement perspective. An algorithm is a budget-limited contractor which chooses among $K$ possible types of offers, \eg different items to procure from vendors, or different tasks to complete in an online labor market. In each round, a new agent arrives, the algorithm chooses an offer and presents it to the customer at an exogenously fixed price. If the offer is accepted, the contractor receives some utility (\ie reward) and spends the corresponding amount of money.

In \emph{dynamic pricing} \citep{BZ09,DynPricing-ec12,BesbesZeevi-or12,Wang-OR14} an algorithm is a seller with limited supply of some product, and chooses a price in each round. If this price is accepted, a sale happens, and algorithm receives revenue and spends inventory. Of our interest is the case when the set of possible prices is small and exogenously fixed, \eg there are a few possible discount levels. Likewise, in \emph{dynamic procurement}  \citep{DynProcurement-ec12,Krause-www13,BwK-focs13-conf},
an algorithm is a budget-limited contractor who continuously procures some product or service. The algorithm chooses a price in each round. If this price is accepted, a transaction happens, so that the algorithm receives an ``item" (\ie reward of $1$) and spends the corresponding amount of money. We focus on the case when there are only a few possible prices, \eg exogenously fixed levels of premium or surcharge.

Our last example concerns fault-tolerance in systems. Consider a system, either physical or computational, which experiments with different possible policies to process incoming requests. In each time step, it chooses one of the possible policies, and observes the outcome (and there are no lingering effects, \eg no persistent ``system state" that changes over time). The outcome consists of utility for performance-as-usual (\ie reward), and penalty for various mistakes or faults. Fault-tolerance requirement is expressed as a a ``budget" on the total penalty accrued by the algorithm.

\section{Confidence bounds in \UCBBwK}
\label{app:spec}
Let us fill in the exact specification of the confidence bounds in the \UCBBwK algorithm. (This is for the sake of completeness only; as pointed out in Preliminaries, these details do not affect our analysis.)

\xhdr{Confidence radius.}
Given an unknown quantity $\mu$ and its estimator $\widehat{\mu}$, a \emph{confidence radius} is an observable high-confidence upper bound on $|\mu-\widehat{\mu}|$. More formally, it is some quantity $\rad\in \R_{\geq 0}$ such that it is computable from the algorithm's observations, and $|\mu-\widehat{\mu}|\leq \rad$ with probability (say) at least $1-\nicefrac{1}{T^3}$. Throughout, the estimator $\widehat{\mu}$ is a sample average over all available
observations pertaining to $\mu$, unless specified otherwise.

Following the prior work on \BwK \citep{DynPricing-ec12,BwK-focs13,AgrawalDevanur-ec14}, we use the confidence radius from \cite{LipschitzMAB-JACM}:
\begin{align}\label{eq:prelims-ConfRad}
\fRad(\widehat{\mu}, N) := \min\rbr{1,\;
    \sqrt{\tfrac{\myGamma\; \widehat{\mu}}{\max(1,N)}} + \tfrac{\myGamma}{\max(1,N)}
    },
\text{ where } \myGamma = 3 \cdot \log (KdT),
\end{align}
and $N$ is the number of samples. If $\widehat{\mu}$ is a sample average of $N$ independent random variables with support in $[0,1]$, and $\mu = \E[\mu]$, then
with probability at least $1- (Kdt)^{-2}$ we have
\begin{align}\label{eq:lem:radBound}
    |\widehat{\mu} - \mu| \leq \fRad(\widehat{\mu}, N) \leq 3\; \fRad(\mu, N).
\end{align}
For each arm, we use this confidence radius separately for expected reward of this arm, and expected consumption of each resource.x


\xhdr{Confidence bounds.}
Fix arm $a\neq \nullArm$, round $t$, and resource $j\neq\term{time}$.

Let $S_t(a) = \{ s<t: a_s = a \}$ be the set of all previous rounds in which this arm has been chosen, and let $N_t(a) = |S_t(a)|$. Let
\begin{align}\label{eq:ave-defn}
\textstyle
    \hat{r}_t(a) := \frac{1}{t} \sum_{s \in S_t(a)} r_s(a)
\quad\text{and}\quad
    \hat{c}_{j, t}(a) := \frac{1}{t} \sum_{s \in S_t(a)} c_{j, s}(a)
\end{align}
denote, resp., the sample average of reward and resource-$j$ consumption of this arm so far.

Define the confidence radii $\rad_{0,t}(a)$ and $\rad_{j,t}(a)$ for, resp., expected reward $r(a)$ and resource consumption $c_j(a)$, and the associated upper/lower confidence bounds:
\begin{align}\label{eq:confidencebounds}
r_t^{\pm}(a)
    &= \operatorname{proj}\left(\; \hat{r}_t(a) \pm \rad_{0,t}(a) \;\right),
     &\rad_{0,t}(a) := \fRad(\hat{r}_t(a), N_t(a)),
     \nonumber\\
c_{j, t}^{\pm}(a)
    &= \operatorname{proj}(\;\hat{c}_{j, t}(a) \pm \rad_{j,t}(a) \;),
    &\rad_{j,t}(a) := \fRad(\hat{c}_{j, t}(a), N_t(a)),
\end{align}
where $\operatorname{proj}(x) := \arg\min_{y\in[0,1]} |y-x| $ denotes the projection into $[0,1]$.
Then, the event
\begin{align}\label{eq:cleanEvent-app}
r(a) \in [r_t^-(a),\;r_t^+(a)] \;\text{and}\;
c_j(a)&\in [c^-_{j, t}(a),\;c^+_{j, t}(a)],\quad
\forall a\in[K], j\in [d-1].
\end{align}
holds for each round $t$ with probability (say) at least $1-\frac{\logThm}{T^4}$ \citep{DynPricing-ec12}.

Note that all confidence radii in \eqref{eq:confidencebounds} are upper-bounded by
\begin{align}\label{eq:MaxConfRadUB-app}
\rad_t(a)
    := \fRad(1,N_t(a)),
\end{align}
which is a version of a more standard confidence radius
    $\tildeO(1/\sqrt{N_t(a)})$.

There is no uncertainty on the time resource and the null arm. So, we set
    $\rad_{\term{time},\,t}(\cdot) = 0$ and $c_{\term{time},\,t}^{\pm}(\cdot) = B/T$,
and
    $\rad_{0,t}(\nullArm) = \rad_{j,t}(\nullArm) = r^{\pm}(\nullArm) = c_{j,t}^{\pm}(\nullArm) = 0$.

\section{LP Sensitivity: proof of Lemma~\ref{cl:sensitivity-body}}
\label{app:LP-sensitivity}

We focus on the sensitivity \emph{of the support of the optimal solution}. We build on some well-known results, which we state below in a convenient form (and provide a proof for completeness).  We use the textbook material from \citet{bertsimas1997introduction}.


Throughout this appendix, we consider a best-arm-optimal problem instance with best arm $a^*$. Let $\vec{X}^*$ denote the optimal solution for the linear program~\eqref{lp:primalAbstract}. Recall that the support of $X^*$ is either $\{a^*\}$ or
    $\{a^*,\nullArm\}$.
We consider perturbations in the \emph{rescaled LP}:
	\begin{equation}
		\label{lp:rescaledLP}
		\begin{array}{ll@{}ll}
		\text{maximize} \qquad
		&
		 \vX \cdot  \vec{r} & \text{such that}\\
        & \vX \in [0,1]^K\\
 		&  \vX \cdot \vec{1} = 1  \\
\displaystyle \forall j \in [d-1]  \qquad
&  \vX \cdot \vec{c}_j \leq (\nicefrac{B}{T})(1-\myEta) \\
&\vX \cdot \vec{c}_d \leq \nicefrac{B}{T}.

\end{array}
\end{equation}
Recall that $\vr, \vc_j\in[0,1]^K$ are vectors of expected rewards and expected consumption of resource $j$.  The $d$-th resource is time. The rescaling parameter $\myEta$ is given in \refeq{eq:prelims-eta}.

Let $\OPTLPSC$ denote the value of this LP; it is easy to see that
    $\OPTLPSC = (1-\myEta)\;\OPTLP$.

We observe that $a^*$ is the best arm for the rescaled LP, too, because $\Dmin$ is large enough. Call a distribution over arms \emph{null-degenerate} if its support includes exactly one non-null arm.

\begin{claim}\label{cl:app-rescaledBasis}
The rescaled LP~\eqref{lp:rescaledLP} has a null-degenerate optimal solution with non-null arm $a^*$.
\end{claim}
\begin{proof}
From the theory in \cite[Ch.5]{bertsimas1997introduction}, if the optimal basis to $\LP$~\eqref{lp:primalAbstract} remains \emph{feasible} to the rescaled $\LP$~\eqref{lp:rescaledLP} then the basis is also optimal to this $\LP$. This is because $\LP$~\eqref{lp:rescaledLP} is obtained by a small perturbation to the right-hand side values in $\LP$~\eqref{lp:primalAbstract}. Let $\vec{X}^*$ denote the optimal solution to $\LP$~\eqref{lp:primalAbstract}. From assumption this is a null-degenerate optimal solution. Using the same analysis in \cite[Ch. 4.4]{bertsimas1997introduction} we only have to show that the perturbation is smaller than $X^*(a^*)$. Since the perturbation is $\tfrac{B \myEta}{T} \leq \tfrac{3 \sqrt{B} \log(KTd)}{T}$ while $X^*(a^*) > \tfrac{3 \sqrt{B} \log (KTd) }{T}$, this perturbation does not change the basis. Thus, the rescaled LP has a null-degenerate optimal solution.
\end{proof}

\begin{claim}
	\label{cl:LagRedefined}
	Let $\vec{\lambda}^*$ denote the vector of the optimal dual solution to the $\LP$~\eqref{lp:primalAbstract}. Then
\begin{align}\label{eq:LagRedefined}
	\DGap(a) = \textstyle \frac{T}{B} \sum_{j \in [d]} \lambda^*_j c_j(a) - r(a).
\end{align}
\end{claim}

\begin{proof}
	From Eq.~\eqref{eq:Glag} we have the following.
\begin{align*}
\DGap(a)
    &:=  \mL(\vec{X}^*, \vec{\lambda}^*) - \mL(\vec{X}_a, \vec{\lambda}^*) \\
    &= \textstyle  \vr(\vX^*) - \frac{T}{B} \sum_{j \in [d]} \lambda^*_j
     \;\vc_j(\vX^*)+ \frac{T}{B} \sum_{j \in [d]} \lambda^*_j c_j(a) - r(a).
\end{align*}

Consider the dual of the $\LP$~\eqref{lp:primalAbstract}. It can be seen that the objective of this dual is $\sum_{j \in [d]} \lambda_j$. It follows that
$\OPTLP = \sum_{j \in [d]} \lambda^*_j$
by strong duality \citep[Section 5.2.3]{boyd2004convex}. As proved in \mycite{AdvBwK-focs19}, $\mL(\vec{X}^*, \vec{\lambda}^*) = \OPTLP$. Thus,
\begin{align*}
\textstyle
\sum_{j \in [d]} \lambda^*_j
    = \OPTLP
    = \mL(\vec{X}^*, \vec{\lambda}^*)
	= \vr(\vX^*) - \frac{T}{B} \sum_{j \in [d]} \lambda^*_j
        \;\vc_j(\vX^*) + \sum_{j \in [d]} \lambda^*_j.		
\end{align*}	
Therefore,
    $\vr(\vX^*) = \frac{T}{B} \sum_{j \in [d]} \lambda^*_j \;\vc_j(\vX^*)$,
which implies \eqref{eq:LagRedefined}.
\end{proof}


Claim~\ref{cl:sensitivity-body} easily follows from the following standard result by letting $\delta(a) = \rad_t(a)$.

\begin{theorem}[perturbation]
\label{thm:perturbedColumn}
Posit only one resource other than time (\ie $d=2$).
Consider a perturbation of the rescaled LP~\eqref{lp:rescaledLP}, where the reward vector $\vr$ is replaced with $\vec{\tilde{r}}$, and the consumption vector $\vc_1$ for the non-time resource is replaced with $\tilde{\vc}_1$. Let $\vec{\tilde{X}}^*$ be its optimal solution. Assume
    $0\leq \tilde{\vr}-\vr \leq \vec{\delta}$
and
    $0\leq \vc_1 - \tilde{\vc}_1 \leq \vec{\delta}$,
for some vector
    $\vec{\delta}\in[0,1]^K$.
 Then for each arm $a \neq a^*$,
  \[ \delta(a) > \DGap(a) \quad\text{if}\quad  a \in \supp(\vec{\tilde{X}}^*).\]
\end{theorem}



\begin{proof}
			Let $\lambda^*_1 \geq 0$ denote the dual variable corresponding to the single resource. Note that since $\OPTLP \leq 1$ and the dual vector $\vec{\lambda}^* \geq \vec{0}$ coordinate wise, we have $\lambda^*_1 \leq 1$. From
    \cite[Ch. 5.1]{bertsimas1997introduction}
on local sensitivity when non-basic column of $A$ is changed, we have that the maximum allowable change to any single column $\delta(a) \leq \frac{\tilde{c}(a)}{\lambda^*_1}$ where $\tilde{c}(a)$ is the reduced-cost for the simplex algorithm, as defined in \cite{bertsimas1997introduction}. We will show that $\tilde{c}(a) = \DGap(a)$. Thus, if $\delta(a) \leq \frac{\tilde{c}(a)}{\lambda^*_1} = \frac{\DGap(a)}{\lambda^*_1}$ we have that the basis remains unchanged. Likewise from \citet[Ch. 5]{bertsimas1997introduction}, the maximum allowed perturbation $\delta(a)$ on the reward $r(a)$ for the basis to remain unchanged is $\delta(a) \leq \tilde{c}(a)$. Combining these two we get the ``\emph{if}'' part of the theorem. 
			
			It remains to prove that the reduced cost $\tilde{c}(a) = \DGap(a)$. After converting the linear program to the standard form as required in \cite{bertsimas1997introduction}, the reduced-cost $\tilde{c}(a)$ is given by the expression $\frac{T}{B(1-\myEta)}\sum_{j \in [d]} c_j(a) \tilde{\lambda}^*_j - r(a)$ where $\vec{\tilde{\lambda}}^*$ is the optimal dual solution to LP~\eqref{lp:rescaledLP}. Note that $\vec{\lambda}^*:= \left( \frac{1}{1-\myEta}\right) \vec{\tilde{\lambda}^*}$ is an optimal solution to the dual of the $\LP$~\eqref{lp:primalAbstract}. Thus, plugging it into the definition of reduced cost and combining it with Claim~\ref{cl:LagRedefined} we have that
 \[ \tilde{c}(a) = \frac{T}{B} \sum_{j \in [d]} \lambda^*_j c_j(a) - r(a) = \DGap(a).\qedhere\]
\end{proof}

\newpage
\section{Various technicalities from Sections~\ref{sec:algorithm}}
\label{appx:logRegretSection}

\subsection{Standard tools} \label{appx:techLemmas}
We rely on some standard tools, which we state below for the sake of convenience.

\begin{theorem}[Wald's identity]\label{thm:wald}
Let $X_i:\;i\in\N$ be i.i.d. real-valued random variables, adapted to filtration $\mF_i:\;i\in\N$. Let $N$ be a stopping time relative to the same filtration. Then
		\[
			\E[X_1 + X_2 + \ldots + X_N] = \E[X_i] \cdot \E[N].
		\]
	\end{theorem}

\begin{theorem}[Optimal Stopping Theorem]\label{thm:optStop}
Let $X_i:\;i\in\N$ be a martingale sequence with $\E[X_0] = 0$ adapted to filtration $\mF_i:\;i\in\N$. Let $N$ be a stopping time relative to the same filtration. Then we have that $\E[X_N]  = 0$.
	\end{theorem}

\begin{theorem}[\cite{LipschitzMAB-stoc08,DynPricing-ec12}]
	\label{thm:myAzuma}
	Let $Z_1, Z_2, \LDOTS Z_T$ be a martingale w.r.t. filtration $(\mF_t)_{t\in[T]}$, such that $|Z_t| \leq c$ for all $t \in [T]$. Let
     $\mu := \frac{1}{T} \sum_{t \in [T]} \E[Z_t \mid \mF_{t-1}]$. Then,
			\[
					\textstyle \Pr \left[ \left| \sum_{t \in [T]} Z_t - \mu T \right| > \sqrt{2 \mu T c^2 \ln \frac{T}{\delta}} \right] \leq \delta.
			\]	
\end{theorem}

\subsection{Proof of \refeq{cl:change-B}}
\label{sec:Wald}

	Let $\tau$ denote the stopping time of the algorithm that chooses arm $a^*$ in every time-step, given that the total budget is $B_0$, $T_0$ on the two resources. From definition we have $\REW(a^*\mid B_0, T_0) = \sum_{t \in [\tau]} r_t(a^*)$. Using Wald's identity (Theorem~\ref{thm:wald}), we have that $\E[\REW(a^*\mid B_0, T_0)] = \E[\tau]\; r(a^*)$.
	
Let $B_0$, $T_0$ denote the budget remaining for the two resources. By definition, we have that $\tau \geq T_0$ and $\sum_{t \in [\tau]} c_t(a^*) \geq B_0$. Using the Wald's identity (Theorem~\ref{thm:wald}) we have that $\E[\sum_{t \in [\tau]} c_t(a^*)] = \E[\tau] c(a^*)$. Thus, we have $\E[\tau] \geq \min \left\{ T_0, \tfrac{B_0}{c(a^*)} \right\} \geq \min \left\{ T_0, B_0 \right\}$. Therefore, we obtain the following.
		
		\begin{equation}
			\label{eq:lowerBoundBpSt}
			\E[\REW(a^*\mid B_0, T_0)] = \E[\tau] r(a^*) > \left( \frac{\min \left\{ T_0, B_0 \right\}}{\max\{ \tfrac{B}{T}, c(a^*) \}} \right) r(a^*), \quad \text{and}
		\end{equation}
		\begin{equation}
			\label{eq:upperBoundBpSt}
			\E[\REW(a^*\mid B)] = \E[\tau_{B}] r(a^*) \leq \left( \frac{B}{\max\{ \tfrac{B}{T}, c(a^*) \}} \right) r(a^*).
		\end{equation}
		Combining Equations~\eqref{eq:lowerBoundBpSt} and \eqref{eq:upperBoundBpSt}, we get \refeq{cl:change-B}.

\subsection{Proof of \refeq{cl:change-B-new}}
	\label{appx:logRegretSectionStronger}
	We now modify the above proof to get the tighter lower-bound in Eq.~\eqref{cl:change-B-new}. Let $T_0$, $B_0$ denote the expected remaining time and budget (respectively) and let $\tau$ denote the (random) stopping time of the algorithm that chooses arm $a^*$ in every time-step given $T_0$ time-steps and $B_0$ budget. This implies that we have, $\mathbb{E}[\sum_{t \in [\tau]} c_t(a^*)] \geq B_0$ and $\E[\tau] \geq T_0$. From Theorem~\ref{thm:wald}, this implies that we have $\E[\tau] c(a^*) \geq B_0$ and $\E[\tau] \geq T_0$. This implies that $\E[\tau] \geq \min\{ T_0, \tfrac{B_0}{c(a^*)} \}$.
	
	Similar to Eq.~\eqref{eq:lowerBoundBpSt} and Eq.~\eqref{eq:upperBoundBpSt} we obtain the following.
	
	\begin{equation}
			\label{eq:lowerBoundBpStTight}
			\E[\REW(a^*\mid B_0, T_0)] = \E[\tau] r(a^*) > \min\{ T_0, \tfrac{B_0}{c(a^*)} \} r(a^*), \quad \text{and}
		\end{equation}
		\begin{equation}
			\label{eq:upperBoundBpStTight}
			\E[\REW(a^*\mid B_0 = B, T_0 = T)] = \OPTFD \leq \left( \frac{B}{\max\{ \tfrac{B}{T}, c(a^*) \}} \right) r(a^*).
		\end{equation}
		
	Combining Equations~\eqref{eq:lowerBoundBpStTight} and \eqref{eq:upperBoundBpStTight}, we get \refeq{cl:change-B-new}.

\subsection{Lower bound on Lagrange gap: Proof of \refeq{eq:gLagP}}
\label{app:gLagP}

We will use \refeq{eq:gLagSimplified} and some standard properties of linear programming.

Assume $c(a^*) < \tfrac{B}{T}$. Using complementary slackness theorem on $\LP$~\eqref{lp:primalAbstract}, this implies that $\lambda^*_1 = 0$. Moreover, note that the objective in the dual of $\LP$~\eqref{lp:primalAbstract} is $\lambda^*_0 + \lambda^*_1 = \lambda^*_0$. The optimal value of the primal $\LP$~\eqref{lp:primalAbstract} is $r(a^*)$ since, $X(a^*) = 1$ is the optimal solution to the $\LP$. This implies that $\lambda^*_0 = r(a^*) \geq \frac{\OPTFD}{T}$. Substituting this into \refeq{eq:gLagSimplified} gives the first inequality in \refeq{eq:gLagP}.
		
		Now assume $c(a^*) > \tfrac{B}{T}$. Again, as above complementary slackness theorem on $\LP$~\eqref{lp:primalAbstract}, this implies that $\lambda^*_0 = 0$. Thus, $\Dmin(a) = \frac{T}{B} \cdot \lambda^*_1 \cdot c(a) - r(a)$. Using the dual objective function $\lambda^*_0 + \lambda^*_1 = \lambda^*_1$ combined with strong duality, this implies that $\lambda^*_1 = \tfrac{\OPTLP}{T} \geq \tfrac{\OPTFD}{T}$. Plugging this back into \refeq{eq:gLagSimplified} gives the second inequality in \refeq{eq:gLagP}.
		
\subsection{Martingale arguments: Proof of \refeq{eq:expectedBudgets}}
\label{app:martingale-arguments}


		For the proof of \refeq{eq:expectedBudgets}, we use the well-known theorem on optimal stopping time of martingales (Theorem~\ref{thm:optStop}). Fix an arm $a \in [K]$. For any subset $S \subseteq [T]$ of rounds let $N_S(a)$, $r_S(a)$ and $c_S(a)$ denote the number of times arm $a$ is chosen, the total realized rewards for arm $a$ and the total realized consumption of arm $a$, respectively. Let $\tau$ denote the (random) stopping time of a $\BwK$ algorithm with (random) budget $B$ and time $T$. Then we have the following claim.
		
		\begin{claim}
			\label{clm:OSTClaim}
			For a random stopping time $\tau$, for every arm $a \in [K]$ we have the following.
			\begin{equation}
			\label{eq:OSTreward}
				\E\left[ r_{[\tau]}(a) \right] = r(a) \cdot \E[N_{[\tau]}(a)].
		\end{equation}
		\begin{equation}
			\label{eq:OSTconsumption}
			\E\left[ c_{[\tau]}(a) \right] = c(a) \cdot \E[N_{[\tau]}(a)].
		\end{equation}
		\end{claim}
		
		\begin{proof}
			We will prove the equality in \refeq{eq:OSTreward}; the one in \refeq{eq:OSTconsumption} follows. Consider $r_{[\tau]}(a)$. By definition this is equal to $\sum_{t \in [\tau]} r_t(a) \cdot \mathbb{I}[a_t = a]$. Let $A_t := \mathbb{I}[a_t = a]$ denote the random variable corresponding to the event that arm $a$ is chosen at time $t$. Define the random variable
			\begin{align*}
					Y_t & := \sum_{t' \leq t} A_{t'} r_{t'}(a) - \E_{t'} \left[A_{t'} r_{t'}(a)\right],
			\end{align*}
			where $\E_t[.]$ denotes the conditional expectation given the random variables $A_1, A_2, \ldots, A_{t-1}$. It is easy to see that the sequences $\{ X_t \} _{t \in [\tau]}$, $\{ Y_t \} _{t \in [\tau]}$ and $\{ Z_t \}_{t \in [\tau]}$ forms a martingale sequence. Thus, we will apply the optimal stopping theorem (Theorem~\ref{thm:optStop}) at time $\tau$, we have the following.
			
			\begin{equation}
			\label{eq:optStopReward}
				\E\left[ Y_{\tau} \right] = \E\left[ \sum_{t' \leq \tau} A_{t'} r_{t'}(a) \right] - \E\left[ \sum_{t' \leq \tau} \E_{t'}\left[A_{t'} r_{t'}(a)\right]\right] = 0.
			\end{equation}
			Consider the term $\E\left[ \sum_{t' \leq \tau} \E_{t'}\left[A_{t'} r_{t'}(a)\right]\right]$ in \refeq{eq:optStopReward}. This can be simplified to\\ $\E\left[ \sum_{t' \leq \tau} r(a) \cdot \Pr[a_{t'} = a]\right]$. Consider the following random variable
			\begin{align*}
				Z_t := \sum_{t' \leq t} 	\Pr[a_{t'} = a] - \E_{t'}[\Pr[a_{t'} = a]].
			\end{align*}
			Note that $\sum_{t' \leq t} \E_{t'}[\Pr[a_{t'} = a]] = N_{[t]}(a)$. Thus, using Theorem~\ref{thm:optStop} on the sequence $Z_t$ at the stopping time $\tau$, we obtain $\E\left[ \sum_{t' \leq \tau} \Pr[a_{t'} = a]\right] = \E[N_{[\tau]}(a)]$.
						
			Thus, the term $\E\left[ \sum_{t' \leq \tau} \E_{t'}\left[A_{t'} r_{t'}(a)\right]\right]$ in \refeq{eq:optStopReward} simplifies to $r(a) \cdot N_{[\tau]}(a)$ which gives the required equality in \refeq{eq:OSTreward}.
		\end{proof}

		We will now use Claim~\ref{clm:OSTClaim} to prove \refeq{eq:expectedBudgets}. Recall that $\REW(a\mid B(a), T(a))$ denotes the total contribution to the reward by the $\BwK$ algorithm by playing arm $a$ with a (random) resource consumption of $B(a)$ and time steps of $T(a)$. Let $\tau$ be the (random) stopping time of this algorithm. By definition we have that $N_{[\tau]}(a) = T(a)$. Thus, $\E[N_{[\tau]}(a)] = \E[T(a)$. From \refeq{eq:OSTconsumption}, we also have that $\E[N_{[\tau]}(a)] = \tfrac{\E\left[ c_{[\tau]}(a) \right]}{c(a)}$. From the definition of $B(a)$ we have, $B(a) = c_{[\tau]}(a)$ and thus, $\E[B(a)] = \E[c_{[\tau]}(a)]$. Thus, this implies that $\E[N_{[\tau]}(a)] = \min\{ T(a), \tfrac{\E[B(a)]}{c(a)} \}$.
		
		Consider $\E[\REW(a)] = \E[\REW(a\mid B(a), T(a))]$.
		\begin{align}
				\E[\REW(a\mid B(a), T(a))] & = \E\left[ r_{[\tau]}(a) \right] & \nonumber \\
				& =  r(a) \cdot \E[N_{[\tau]}(a)] & \EqComment{From \refeq{eq:OSTreward}} \nonumber \\
				& = r(a) \cdot \min\{ T(a), \tfrac{\E[B(a)]}{c(a)} \} & \label{eq:plugInReward}
		\end{align}
					
		Now, consider $\LP(a\mid \E[B(a)], \E[T(a)])$. This value is equal to,
		\begin{align*}
			\E[\REW(a\mid \E[B(a)], \E[T(a)])] & = \frac{r(a)}{\max\{ \E[B(a)]/\E[T(a)], c(a)\}} \cdot \tfrac{\E[B(a)]}{\E[T(a)]} \\
			& = r(a) \cdot \min \left\{ \E[T(a)], \tfrac{\E[B(a)]}{c(a)} \right\}.
		\end{align*}
		Note that the last equality is same as the RHS in \refeq{eq:plugInReward}.

\newpage
\section{Proof of Theorem~\ref{thm:generalLB}: generic $\sqrt{T}$ lower bound}
\label{sec:LB-generic}

\xhdr{Preliminaries.}
We rely on a  well-known information-theoretic result for multi-armed bandits: essentially, no algorithm can reliably tell apart two bandit instances at time $T$ if they differ by at most $O(1/\sqrt{T})$.%
\footnote{This strategy for proving lower bounds in multi-armed bandits goes back to \citet{bandits-exp3}. Lemma~\ref{lem:bestArm} is implicit in \citet{bandits-exp3}, see \citet[Lemma 2.9]{slivkins-MABbook} for exposition.}
We formulate this result in a way that is most convenient for our applications.

\begin{lemma}
\label{lem:bestArm}
Consider multi-armed bandits with Bernoulli rewards.
Fix $\eps>0$ and two problem instances $\mI,\mI'$ such that
the mean reward of each arm differs by at most $\eps$ between $\mI$ and $\mI'$.
Suppose some bandit algorithm outputs distribution $\vY_t$ over arms at time $t \leq \nicefrac{c}{\eps^2}$, for a sufficiently small absolute constant $c$. Let $H$ be an arbitrary Lebesgue-measurable set of distributions over arms. Then either
    $\Pr[\vec{Y}_t \in H \mid \mJ_{t} = \mI] > \nicefrac{1}{4}$ or $\Pr[\vec{Y}_t \notin H \mid \mJ_{t} = \mI'] > \nicefrac{1}{4}$ holds.
\end{lemma}

Applying Lemma~\ref{lem:bestArm} to bandits with knapsacks necessitates some subtlety. First, the rewards in the lemma will henceforth be called \emph{\LBrewards}, as they may actually correspond to consumption of a particular resource. Second, while a \BwK algorithm receives multi-dimensional feedback in each round, the feedback other than the \LBrewards will be the same (in distribution) for both problem instances, and hence can be considered a part of the algorithm. Third, distribution $\vY_t$ will be the conditional distribution over arms chosen by the \BwK algorithm in round $t$ given the algorithm's observations so far; we will assume this without further mention. Fourth, we will need to specify the set $H$ of distributions (which will depend on a particular application).

	Consider the rescaled $\LP$~\eqref{lp:rescaledLP} with $\myEta := 6 \ast \OPTLP \sqrt{\frac{\log dT}{B}}$; we use this $\myEta$ throughout this proof. Let $\OPTLPSC$ be the value of this LP. We prove the lower bound using $\OPTLPSC$ as a benchmark. This suffices by the following claim from prior work:
\footnote{Claim~\ref{claim:OPTFDLB} is a special case of Lemma 8.6 in \citet{AdvBwK-focs19} for $\tau^* = T$ and the reward/consumption for each arm, each resource and each time-step replaced with the mean reward/consumption.}
		
\begin{claim}[\citet{AdvBwK-focs19}]\label{claim:OPTFDLB}
$\OPTLPSC \leq \OPTFD$ for $\myEta := 6 \cdot \OPTLP \sqrt{\frac{\log dT}{B}}$.
\end{claim}

\xhdr{Problem instances.}
Let $\vr(a)$ and $\vc(a)\in [0,1]^d$ be, resp., the mean reward and the mean resource consumption vector for each arm $a$ for instance $\mI_0$. Let $\eps = \cLB/\sqrt{T}$. 	
	
Problem instances $\mI,\mI'$ are constructed as specified in the proof sketch; we repeat it here for the sake of convenience. For both instances, the rewards of each non-null arm $a\in \{A_1,A_2\}$ are deterministic and equal to $r(a)$. Resource consumption vector for arm $A_1$ is deterministic and equals $\vc(A_1)$. Resource consumption vector of arm $A_2$ in each round $t$, denoted $\vc_{(t)}(A_2)$, is a carefully constructed random vector whose expectation is $c(A_2)$ for instance $\mI$, and slightly less for instance $\mI'$. Specifically,
    $\vc_{(t)}(A_2) = \vc(A_2)\cdot W_t/(1-\cLB) $,
where $W_t$ is an independent Bernoulli random variable which correlates the consumption of all resources. We posit $\E[W_t] = 1-\cLB$ for instance $\mI$, and $\E[W_t] = 1-\cLB-\eps$ for instance $\mI'$.


\xhdr{Main derivation.}
	From the premise of the theorem (\refeq{eq:cor:bestArmOptimal}), problem instance $\mI$ admits an optimal solution $\vX^*$ that is substantially supported on both non-null arms. Let $\vX^*_{\mI}$, $\vX^*_{\mI'}$ denote the optimal solutions to the scaled LP, instantiated for instances $\mI, \mI'$ respectively.

The proof proceeds as follows. We first prove certain properties of distributions $\vX^*_{\mI}$ and $\vX^*_{\mI'}$. We then use these properties and apply Lemma~\ref{lem:bestArm} with suitable quasi-rewards to complete the proof of the lower-bounds.
		
		Since we modify the mean consumption of all resources for one arm in $\mI'$ this implies that $\vX^*_{\mI} \neq \vX^*_{\mI'}$. From assumption~\ref{ass:LBAss}-\eqref{boundLagrange} we have that $\Dmin \geq \cLB/\sqrt{T}$. From the premise of the theorem, we have that the mean vector of consumptions for the resources $j \in [d]$ are all linearly independent. Thus, we can apply sensitivity theorem~\ref{thm:perturbedColumn} to conclude that the support of the solution $\vX^*_{\mI'}$ is same as $\vX^*_{\mI}$.

		Moreover, from the linear independence of the consumption vectors and \refeq{eq:cor:bestArmOptimal}. combined with standard $\LP$ theory (see chapter 4 on duality in \mycite{bertsimas1997introduction}) we have that there exists a resource $j^* \in [d]$ such that the optimal solution $\vX^*_{\mI}$ satisfies the resource constraint with equality.

		In what follows, we denote the vector $\vec{c}$ as a shorthand for $\vec{c}_{j^*}$ (\emph{i.e.,} we drop the index $j^*$). Note that from the perturbation we have that $c(A_1) < c(A_2)$. Thus, for some $\delta > 0$ we have $X^*_{\mI'}(A_1) = X^*_{\mI}(A_1) - \delta$ and $X^*_{\mI'}(A_2) = X^*_{\mI}(A_2) + \delta$. Let $\|  \vec{X} \|$ denote the $\ell_1$-norm of a given distribution $\vec{X}$. Thus, we have
			\begin{equation}
				\label{eq:diffDist}
				\|  \vX^*_{\mI} - \vX^*_{\mI'} \| = 2 \delta.
			\end{equation}
		
Given any distribution $\vY$ over the arms, let 
\begin{align}\label{eq:LPgap-defn}
\scaledV(\vY)
    := \textstyle (1-\myEta)\cdot \nicefrac{B}{T}\;\cdot
        r(\vY)  / \rbr{\max_{j\in [d]} c_j(\vY)}.
\end{align}
This is the value of $\vY$ in the rescaled LP~\eqref{lp:rescaledLP}, where $\vY$ itself is rescaled to make it LP-feasible (and as large as possible). Note that 
    $\scaledV(\vY) = (1-\myEta)\,V(\vY)$,
where $V(\vY)$ is the value of the original LP, as defined in \eqref{eq:LPgap-defn}. Also, $\OPTLPSC = \sup_{\vY} \scaledV(\vY)$.

By a slight abuse of notation, let $\scaledV(\vY), \scaledV'(\vY)$ be the value of $\scaledV(\vY)$ corresponding to instances $\mI$ and $\mI'$ respectively.

We use the following two claims in the proof of our lower-bound. Claim~\ref{clm:badForI} states that if a distribution is close to the optimal distribution for instance $\mI$ then it is also far from the optimal distribution for $\mI'$. Claim~\ref{clm:distToValue} states that if a distribution is far from the optimal distribution, then playing from that distribution also incurs large instantaneous regret. Both claims have nothing to do with particular algorithms.

		\begin{claim}\label{clm:badForI}
			Fix distribution $\vY \in \Delta^3$ and $\epsilon < 1$. If $\| \vX^*_{\mI} - \vY \| < \fepsilon$ then $\| \vX^*_{\mI'} - \vY \| \geq \gepsilon$.
		\end{claim}
		
		\begin{claim}\label{clm:distToValue}
			Fix distribution $\vY \in \Delta^3$ and $\epsilon < 1$. If $\| \vX^*_{\mI} - \vY \| \geq \fepsilon$ then $\scaledV(\vX^*_{\mI}) - \scaledV(\vY) \geq \tfepsilon$. Likewise, if $\| \vX^*_{\mI'} - \vY \| \geq \gepsilon$ then $\scaledV'(\vX^*_{\mI'}) - \scaledV'(\vY) \geq \gfepsilon$.
			
		\end{claim}

		We now invoke Lemma~\ref{lem:bestArm} with the quasi-rewards at each time-step determined by the consumption of the resource $j^*$.
		
		Define the set,
		\begin{equation}
			\label{eq:defnH}
			\mH := \left \{ \vec{Y}: \| \vX^*_{\mI} - \vY \| \geq \fepsilon \right \},
		\end{equation}
		to complete the proof Theorem~\ref{thm:generalLB}. Consider an arbitrary algorithm $\ALG$. We consider two cases: $\mJ = \mI$ and $\mJ = \mI'$, which denote the instance that satisfies the conclusion of this lemma for at least $\frac{T}{2}$ rounds for $T := \frac{\cLB}{\eps^2}$.
		
		Let $\mJ = \mI$. Let $\mathcal{T}$ denote the set of time-steps $t \in [T]$ such that $\mJ_t = \mI$ and $\vec{Y}_t \in \mH$.
		Then, the expected regret of \ALG can be lower-bounded by,		
\begin{align*}
\E\sbr{ \sum_{t \in \mathcal{T}} \scaledV(\vX^*_\mI) - \scaledV(\vec{Y}_t)  }
    &=  \E\sbr{
            \sum_{t \in \mathcal{T}:\; \| \vX^*_{\mI} - \vY_t \| \geq \fepsilon} \scaledV(\vX^*_\mI) - \scaledV(\vec{Y}_t) }
    	& \EqComment{by \refeq{eq:defnH}} \\
    & \geq \textstyle  \E\sbr{ \sum_{t \in \mathcal{T}} \; \tfepsilon }
    	& \EqComment{by \refeq{clm:distToValue}} \\
			&  \geq \nicefrac{T}{4} \cdot \tfepsilon
    		& \EqComment{by Lemma~\ref{lem:bestArm}}\\
    & \geq
        O\rbr{ \cLB^4 \cdot \sqrt{T}}.
        & \EqComment{Since $\epsilon = \tfrac{\cLB}{\sqrt{T}}$}
\end{align*}

We use a similar argument when $\mJ = \mI'$. Let $\mathcal{T}'$ denote the set of time-steps $t \in [T]$ such that $\mJ_t = \mI'$ and
    $\| \vX^*_{\mI'} - \vY_t \| \geq \gepsilon$.
The expected regret of \ALG can be lower-bounded by,
\begin{align*}
\E \sbr{ \sum_{t \in \mathcal{T}'}  \scaledV'(\vX^*_{\mI'}) - \scaledV'(\vec{Y}_t) }
	& = \E\sbr{
        \sum_{t \in \mathcal{T}':\; \| \vX^*_{\mI'} - \vY_t \| \geq \gepsilon}
            \scaledV'(\vX^*_{\mI'}) - \scaledV'(\vec{Y}_t) } & \\
	& \geq \E\sbr{
        \sum_{t \in \mathcal{T}':\; \| \vX^*_{\mI} - \vY_t \| < \fepsilon}
            \scaledV'(\vX^*_{\mI'}) - \scaledV'(\vec{Y}_t) }
		& \EqComment{by Claim~\ref{clm:badForI}} \\
	& = \E\sbr{
        \sum_{t \in \mathcal{T}':\; \vY_t \notin \mH}
            \scaledV'(\vX^*_{\mI'}) - \scaledV'(\vec{Y}_t) }
		& \EqComment{by \refeq{eq:defnH}} \\
    & \geq  \E\sbr{ \sum_{t\in[T]:\; \vY_t \notin \mH}  \gfepsilon }
    		& \EqComment{by \refeq{clm:distToValue}} \\
			&  \geq \nicefrac{T}{4} \cdot \gfepsilon
    		& \EqComment{by Lemma~\ref{lem:bestArm}}\\
    & \geq
        O\rbr{ \cLB^4 \cdot \sqrt{T}}. & \EqComment{Since $\epsilon = \tfrac{\cLB}{\sqrt{T}}$}.
\end{align*}

\xhdr{Proof of Claim~\ref{clm:badForI}.}
			Let $c(A_1), c(A_2)$ denote the expected consumption of arms $A_1$ and $A_2$ respectively in instance $\mI$. Define $\zeta := \frac{\eps c(A_1)}{1-\cLB}$. By definition, this implies that the expected consumption of arm $A_2$ in instance $\mI'$ is $c(A_2) - \zeta$. Additionally, since the support contains two arms, we have that the following holds: $c(A_1) X^*_{\mI}(A_1) + c(A_2) X^*_{\mI}(A_2) = B/T*(1-\myEta)$ and $c(A_1) X^*_{\mI'}(A_1) + c(A_2) X^*_{\mI'}(A_2) - \zeta X^*_{\mI'}(A_2) = B/T*(1-\myEta)$. Thus, we have
				\[
					c(A_1)X^*_{\mI}(A_1) + c(A_2) X^*_{\mI}(A_2) = c(A_1) X^*_{\mI}(A_1) + c(A_2) X^*_{\mI}(A_2) + \delta ( C(A_2) - c(A_1) - \zeta) - \zeta X^*_{\mI}(A_2).
				\]
				Rearranging and using the assumptions in \ref{ass:LBAss}, we get that
				\begin{equation}
					\label{eq:deltaDef}
					\delta = \frac{\zeta X^*_{\mI}(A_2)}{c(A_2) - c(A_1)- \zeta} \geq \frac{\epsilon \cLB}{1-\cLB} \cdot \frac{\cLB}{1-2\cLB - \tfrac{\epsilon \cdot \cLB}{1-\cLB}} \geq \gepsilon.
				\end{equation}
	
		Consider $\| \vX^*_{\mI'} - \vY \|$. This can be rewritten as
		\begin{align*}
			& = \| \vX^*_{\mI'} - \vY - \vX^*_{\mI} + \vX^*_{\mI} \| & \\
			& \geq |  \| \vX^*_{\mI'} - \vX^*_{\mI} \| - \| \vX^*_{\mI} - \vY \| | & \EqComment{Triangle inequality}\\
			& \geq 2 \delta - \fepsilon & \EqComment{Premise of the claim and \refeq{eq:diffDist}}\\
			& \geq \gepsilon. & \EqComment{From  \refeq{eq:deltaDef}}
		\end{align*}

\xhdr{Proof of Claim~\ref{clm:distToValue}.}
		We will prove the statement $\| \vX^*_{\mI} - \vY \| \geq \fepsilon \implies \scaledV(\vX^*_{\mI}) - \scaledV(\vY) \geq \tfepsilon$. The exact same argument holds by replacing $\vX^*_{\mI}$ with $\vX^*_{\mI'}$ and $\scaledV(.)$ with $\scaledV'(.)$.
		
		Consider $ \scaledV(\vX^*_{\mI}) - \scaledV(\vY)$. By definition, this equals,
		\begin{equation}
			\label{eq:diffReg}
			r(\vX^*_{\mI}) - \frac{r(\vY)}{\max \{\tfrac{B'}{T}, c(\vY) \}} \cdot \frac{B'}{T},
		\end{equation}
		where $B'$ is the scaled budget.
		
		We have two cases. In case 1, let $\max \{\tfrac{B'}{T}, c(\vY) \} = \frac{B'}{T}$. Thus, \refeq{eq:diffReg} simplifies to,
		\begin{align*}
			 & =  r(\vX^*_{\mI}) -  r(\vY) & \\
			 & = r(A_1) [X^*_{\mI}(A_1) - Y(A_1)] + r(A_2)  [X^*_{\mI}(A_2) - Y(A_2)] &
		\end{align*}
		Note that since $\max \{\tfrac{B'}{T}, c(\vY) \} = \frac{B'}{T}$, this implies that $Y(\nullArm) = 0$. Since $\vX^*_{\mI}$ is an optimal solution and $r(A_2) > r(A_1)$, this implies that we have $Y(A_1) = X^*_{\mI}(A_1) + \zeta$ and $Y(A_2) = X^*_{\mI}(A_2) - \zeta$. Thus, we have,
		\begin{align*}
				r(A_1) [X^*_{\mI}(A_1) - Y(A_1)] + r(A_2)  [X^*_{\mI}(A_2) - Y(A_2)] & \geq [r(A_2) - r(A_1)] \zeta \\
				& \geq \cLB \cdot \| \vX^*_{\mI} - \vY \|/2 \\
				& \geq \tfepsilon.
		\end{align*}
		Consider case 2 where $\max \{\tfrac{B'}{T}, c(\vY) \} = c(\vY)$. Then, \refeq{eq:diffReg} simplifies to,
		\begin{align*}
			 & =  r(\vX^*_{\mI}) -  \tfrac{B'}{T} \cdot \tfrac{r(\vY)}{c(\vY)} & \\
			 & \geq r(\vX^*_{\mI}) -  \max_{\vY \in \Delta_3: \| \vX^*_{\mI} - \vY \| \geq \gepsilon} \tfrac{B(1-\myEta)}{T} \cdot \tfrac{r(\vY)}{c(\vY)}
		\end{align*}
		The maximization happens when the distribution $\vY$ is such that $Y(A_1) = X^*_{\mI} - \gepsilon/2$ and $Y(A_2) = X^*_{\mI} - \gepsilon/2$. Plugging this into the expression we get the RHS is at least,		
		\begin{align*}
			& \geq  r(\vX^*_{\mI}) - \tfrac{B(1-\myEta)}{T} \cdot \frac{r(\vX^*_{\mI}) + \gepsilon/2 \cdot (r(A_2) - r(A_1))}{c(\vX^*_{\mI}) + \gepsilon/2 \cdot (c(A_2) -c(A_1))} & \\
			& \geq  r(\vX^*_{\mI}) - \cLB (1-\myEta) \cdot \frac{r(\vX^*_{\mI}) + \gepsilon/2 \cdot (r(A_2) - r(A_1))}{c(\vX^*_{\mI}) + \gepsilon/2 \cdot (c(A_2) -c(A_1))} & \\
			& \geq  r(\vX^*_{\mI}) - (1-\myEta) \cdot \frac{r(\vX^*_{\mI}) + \gepsilon/2 \cdot (r(A_2) - r(A_1))}{1 + \gepsilon/2 } & \\
			& \geq \tfrac{\myEta}{2} \cdot r(\vX^*_{\mI}) \geq \epsilon \cdot \tfrac{\cLB^3}{2}.
		\end{align*}
		
		The last two inequality follows from Assumption~\ref{ass:LBAss}-\eqref{boundOPT}, the value of $\myEta$ and the fact that $\epsilon = \tfrac{\cLB}{\sqrt{T}}$, respectively.
		Combining the two cases we get the claim. 

\section{Proof of Theorem~\ref{thm:LB-root}(b): $\sqrt{T}$ lower bound for $d>2$}
\label{sec:LB-cor}
\OMIT{
We now derive Theorem~\ref{cor:bestArmOptimal} and Theorem~\ref{cor:multipleResources} as corollaries of Theorem~\ref{thm:generalLB}. It suffices to show that they satisfy the premise of Theorem~\ref{thm:generalLB} (\ie satisfies \refeq{eq:per-resource-vector}). The corollaries then follow from the conclusion of Theorem~\ref{thm:generalLB}.

\xhdr{Proof of Theorem~\ref{cor:bestArmOptimal}.} We have $d=2$ resources (including time). From \refeq{eq:cor:bestArmOptimal} the optimal basis corresponding to the optimal solution $\vX$ consists of two arms. From standard $\LP$ duality theory as used in the proof of Theorem~\ref{thm:generalLB} (see also chapter 4 on duality in \cite{bertsimas1997introduction}), this implies that there exists exactly \emph{two} linearly independent set of constraints in the scaled $\LP$~\eqref{lp:rescaledLP} that are satisfied with equality for the optimal solution $\vX$. Since there are exactly two constraints in the scaled $\LP$, this implies that all constraints are linearly independent (\ie \refeq{eq:per-resource-vector} is satisfied).

\xhdr{Proof of Theorem~\ref{cor:multipleResources}.}
} 

We first show that for any given instance $\mI_0$, for a given $0 < \delta_1 \leq \mathcal{O}\left(\frac{1}{\sqrt{T}}\right)$ we can obtain a $\delta_1$-perturbation of this instance, denoted by $\mI_0'$, that satisfies \refeq{eq:cor:bestArmOptimal}. Given instance $\mI_0$ we construct the $\delta_1$-perturbation as follows. We construct instance $\mI_0'$ by decreasing the mean consumption on arm $A_i$ and resource $j$ by $\zeta_1^j$. We keep the mean rewards the same. Let $\vX$ denote the optimal solution to instance $\mI$. As a notation we denote the matrix $\vec{C} \in [0, 1]^{d \times 3}$ as the matrix of mean consumption. Let $\vec{B}$ denote the sub-matrix of $\vec{C}$ such that, $\vX$ satisfies the constraints in the scaled $\LP$~\eqref{lp:rescaledLP} with equality. Thus, we have $\vec{C} \cdot \vX = \vec{b}$, where every co-ordinate of $b$ is $\frac{B (1-\myEta)}{T}$. Thus, the perturbation is equivalent to perturbing the vector $\vec{b}$, such that the $j^{th}$ entry has an additive perturbation of $\zeta^j$. From Proposition 3.1 in \mycite{megiddo1988perturbation}, this linear program has a non-degenerate primal optimal solution, in the sense that it satisfies \refeq{eq:cor:bestArmOptimal}.

	Next, we show that given an instance $\mI_0'$ we can obtain a $\delta_2$ perturbation of $\mI_0'$ for a given $0 < \delta_2 \leq \mathcal{O}\left(\frac{1}{\sqrt{T}}\right)$, such that the consumption vectors are linearly independent.
Define a random matrix $\vec{D} \in [-\zeta_2, \zeta_2]^{d \times 3}$ such that every entry in $\vec{D}$ is generated uniformly at random from the set $[-\zeta_2, \zeta_2]$. We claim that the vectors $\vec{c}_j - \vec{d}_j$ are all linearly independent, where $\vec{d}_j$ is the $j^{th}$ row of $\vec{D}$ with probability at least $0.6$. In other words, decreasing each of the mean consumption by a uniformly random value chosen from the set $[-\zeta_2, \zeta_2]$ implies that  there exists a realization of $\vec{D}$ such that the vectors $\vec{c}_j - \vec{d}_j$ are all linearly independent.
	
	The proof of this claim proceeds as follows. As before define $\vec{C} \in [0, 1]^{d \times 3}$ to be the matrix of mean consumption. From definition of linear independence we need to show that the smallest singular value of the matrix $\vec{C} - \vec{D}$ is non-zero. Note that every entry in the matrix $\vec{C} - \vec{D}$ is chosen independently. Thus, using the bound on the probability of singularity in Theorem 2.2 of \mycite{bourgain2010singularity} we have that the probability that the smallest singular value is $0$ is at most $\frac{1}{2\sqrt{2}}$. Thus, with probability at least $1-\frac{1}{2\sqrt{2}} > 0.6$ we have that the matrix $\vec{C}-\vec{D}$ is singular.
	
	Thus, for $\delta := \delta_1 + \delta_2$, we have that there exists a $\delta$-perturbed instance $\NewI$, that satisfies all the assumptions in \ref{ass:LBAss} and linear independence condition required in the premise of Theorem~\ref{thm:generalLB}. 

\end{appendices}



\end{document}